\documentclass{article}
\usepackage[accepted]{icml_fmwild}
\icmltitlerunning{\bf Geometric Median ($\gm$) Matching for Robust Data Pruning}
\newcommand{\gm}{\textsc{Gm}}
\newcommand{\gmm}{\textsc{Gm Matching}}
\newcommand{\Mu}{\bm{\mu}}
\PassOptionsToPackage{cmyk}{xcolor}
\usepackage{natbib}
\usepackage[utf8]{inputenc} % allow utf-8 input
\usepackage[T1]{fontenc}    % use 8-bit T1 fonts
\usepackage{url}            % simple URL typesetting
\usepackage{booktabs}       % professional-quality tables
\usepackage{amsfonts}       % blackboard math symbols
\usepackage{nicefrac}       % compact symbols for 1/2, etc.
\usepackage{microtype}      % microtypography
\usepackage{xcolor}         % colors
\usepackage{wrapfig}
\usepackage{graphicx}
\usepackage[algo2e]{algorithm2e}
\usepackage{comment}
\usepackage{subfig}
\usepackage{enumitem}
\setlist{leftmargin=3mm}
\usepackage{multirow}
\usepackage{tcolorbox}
\usepackage{tikz}
\usepackage{listings}
\usepackage{amsmath,amsthm,amssymb,amsfonts}
\usepackage{algorithm}
\usepackage{algorithmic}
\usepackage{comment}
\usepackage{bm}
\usepackage{pifont}

\theoremstyle{plain}
\newtheorem{theorem}{Theorem}[]

\newtheorem{lemma}[]{Lemma}

\newtheorem{definition}{Definition}[]

\def\1{\mathbf{1}}

\def\vtheta{{\boldsymbol{\theta}}}

\def\vm{{\mathbf{m}}}

\def\vx{{\mathbf{x}}}

\def\vz{{\mathbf{z}}}

\def\mB{{\mathbf{B}}}

\DeclareMathAlphabet{\mathsfit}{\encodingdefault}{\sfdefault}{m}{sl}
\SetMathAlphabet{\mathsfit}{bold}{\encodingdefault}{\sfdefault}{bx}{n}

\def\gB{{\mathcal{B}}}

\def\gD{{\mathcal{D}}}

\def\gG{{\mathcal{G}}}
\def\gH{{\mathcal{H}}}

\def\gM{{\mathcal{M}}}

\def\gO{{\mathcal{O}}}

\def\gS{{\mathcal{S}}}

\def\gY{{\mathcal{Y}}}

\def\sR{{\mathbb{R}}}

\newcommand{\E}{\mathbb{E}}

\DeclareMathOperator*{\argmax}{arg\,max}
\DeclareMathOperator*{\argmin}{arg\,min}

\usepackage{color, colortbl}
\usepackage{bbm}
\usepackage{bm}
\usepackage{hyperref}
\usepackage[capitalize,noabbrev]{cleveref}

\lstdefinestyle{mystyle}{
    backgroundcolor=\color{backcolour},   
    commentstyle=\color{codegreen},
    keywordstyle=\color{magenta},
    numberstyle=\tiny\color{codegray},
    stringstyle=\color{codepurple},
    basicstyle=\ttfamily\footnotesize,
    breakatwhitespace=false,         
    breaklines=true,                 
    captionpos=b,                    
    keepspaces=true,                 
    numbers=left,                    
    numbersep=5pt,                  
    showspaces=false,                
    showstringspaces=false,
    showtabs=false,                  
    tabsize=2,
    frame=single,
}
\usepackage[colorinlistoftodos,prependcaption]{todonotes}
\usepackage{soul}
\definecolor{LightCyan}{rgb}{0.9,1,1}
\definecolor{light-gray}{gray}{0.8}
\definecolor{codegreen}{rgb}{0,0.6,0}
\definecolor{codegray}{rgb}{0.5,0.5,0.5}
\definecolor{codepurple}{rgb}{0.58,0,0.82}
\definecolor{backcolour}{rgb}{0.95,0.95,0.92}

\begin{document}

\twocolumn[
    \icmltitle{\bf Geometric Median ($\gm$) Matching for Robust Data Pruning}
    
    \icmlsetsymbol{equal}{*}
    \begin{icmlauthorlist}
    \icmlauthor{Anish Acharya}{x}
    \icmlauthor{Inderjit S Dhillon}{x,y}
    \icmlauthor{Sujay Sanghavi}{x,z}
    % \icmlauthor{Firstname4 Lastname4}{sch}
    % \icmlauthor{Firstname5 Lastname5}{yyy}
    % \icmlauthor{Firstname6 Lastname6}{sch,yyy,comp}
    % \icmlauthor{Firstname7 Lastname7}{comp}
    % %\icmlauthor{}{sch}
    % \icmlauthor{Firstname8 Lastname8}{sch}
    % \icmlauthor{Firstname8 Lastname8}{yyy,comp}
    %\icmlauthor{}{sch}
    %\icmlauthor{}{sch}
    \end{icmlauthorlist}
    \icmlaffiliation{x}{UT Austin}
    \icmlaffiliation{y}{Google}
    \icmlaffiliation{z}{Amazon}
    % \icmlaffiliation{w}{Eidon AI}
    \icmlcorrespondingauthor{Anish Acharya}{anishacharya@utexas.edu}
    \icmlkeywords{Machine Learning, ICML}
    \vskip 0.3in
]
\printAffiliationsAndNotice{}

\begin{abstract}
Large-scale data collections in the wild, are invariably noisy. Thus developing data pruning strategies that remain robust even in the presence of corruption is critical in practice. In this work, we propose Geometric Median ($\gm$) Matching -- a herding style greedy algorithm that yields a $k$-subset such that the mean of the subset approximates the geometric median of the (potentially) noisy dataset. Theoretically, we show that $\gm$ Matching enjoys an improved $\gO(1/k)$ scaling over $\gO(1/\sqrt{k})$ scaling of uniform sampling; while achieving {\bf optimal breakdown point} of {\bf 1/2} even under {\bf arbitrary} corruption. Extensive experiments across several popular deep learning benchmarks indicate that $\gm$ Matching consistently improves over prior state-of-the-art; the gains become more profound at high rates of corruption and aggressive pruning rates; making $\gm$ Matching a strong baseline for future research in robust data pruning.
\end{abstract}

\section{Background}
Data pruning, the (combinatorial) task of downsizing a large training set into a small informative subset~\citep{feldman2020core, agarwal2005geometric, muthukrishnan2005data, har2011geometric, feldman2011unified}, is a promising approach for reducing the enormous computational and storage costs of modern deep learning.
Consequently, a large body of recent works have been proposed to solve the data selection problem. At a high level, data pruning approaches rely on some carefully designed {\bf pruning metrics} and rank the training samples based on the scores and retain a fraction of them as representative samples (super samples), used for training the downstream model. For example,~\citep{xia2022moderate, joshi2023data, sorscher2022beyond} calculate the importance score of a sample in terms of the distance from the centroid of its corresponding class marginal.
Samples closer to the centroid are considered most prototypical (easy) and those far from the centroid are treated as least prototypical (hard). Canonically, similar scoring criterion have been developed in terms of gradients~\citep{paul2021deep}, uncertainty~\citep{pleiss2020identifying}, forgetfulness~\citep{toneva2018empirical}. It is worth noting that the distance-based score is closely related to the uncertainty / gradient forgetting based score. Samples close (far away) to the class centroid are often associated with smaller (larger) gradient norm ; lower (higher) uncertainty; harder (easier) to forget~\citep{paul2021deep, sorscher2022beyond, xia2022moderate}. 

\begin{figure*} 
\centering
\subfloat[\textsc{Easy ($\psi=0$)}]
{\includegraphics[width=0.32\textwidth]{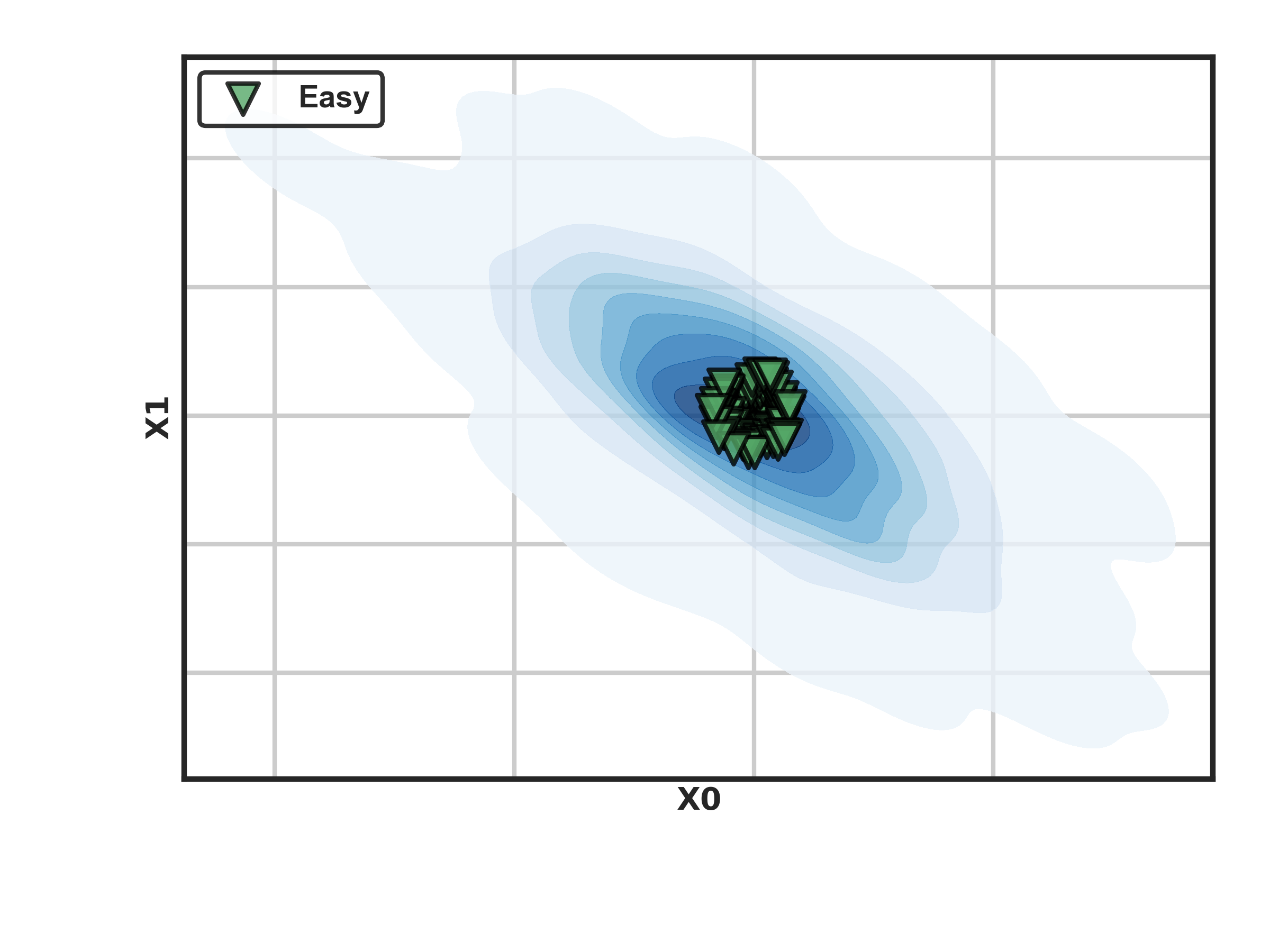}}
\subfloat[\textsc{Easy ($\psi=0.2$)}]
{\includegraphics[width=0.32\textwidth]{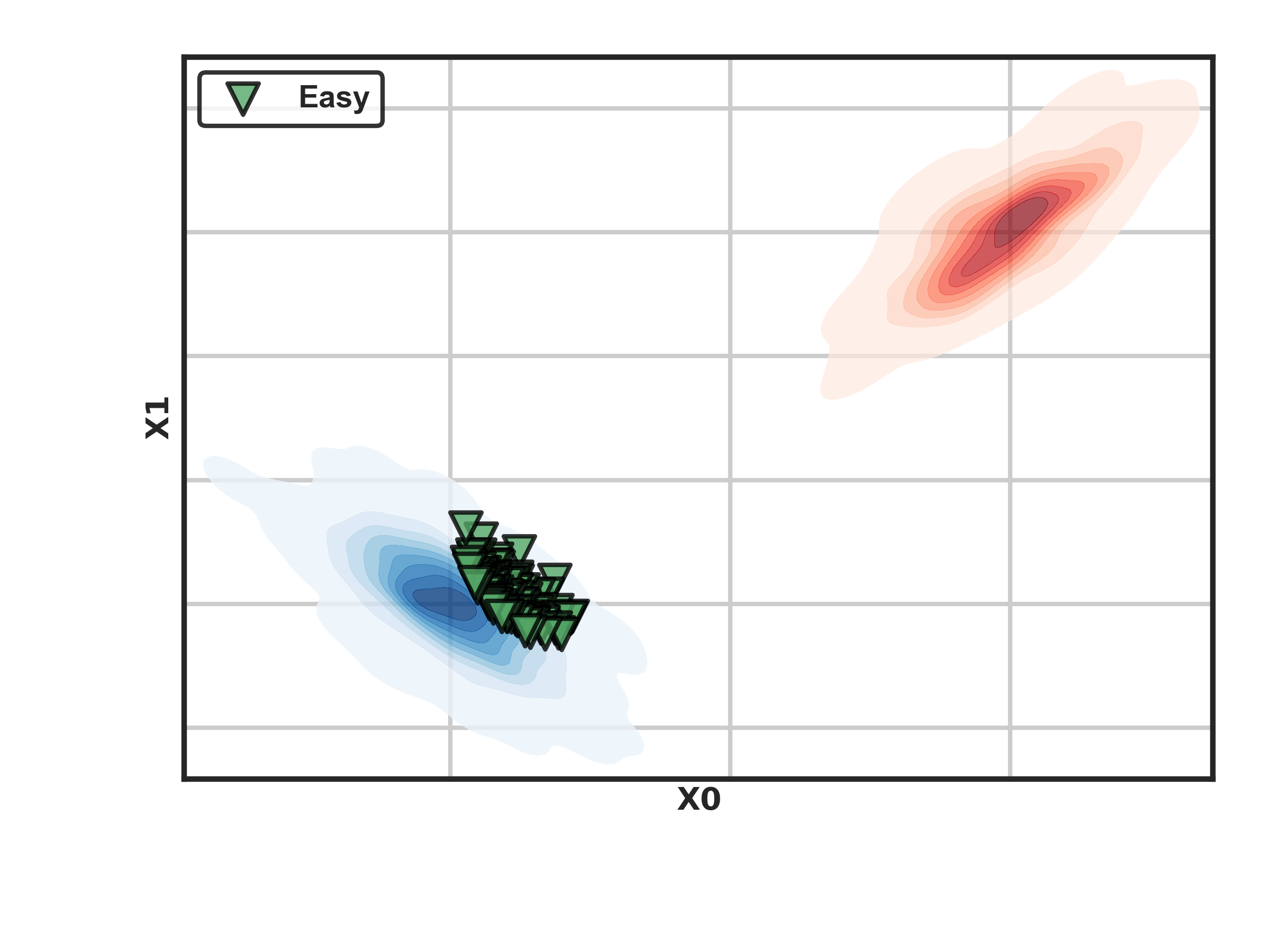}}
\subfloat[\textsc{Easy ($\psi=0.45$)}]
{\includegraphics[width=0.32\textwidth]{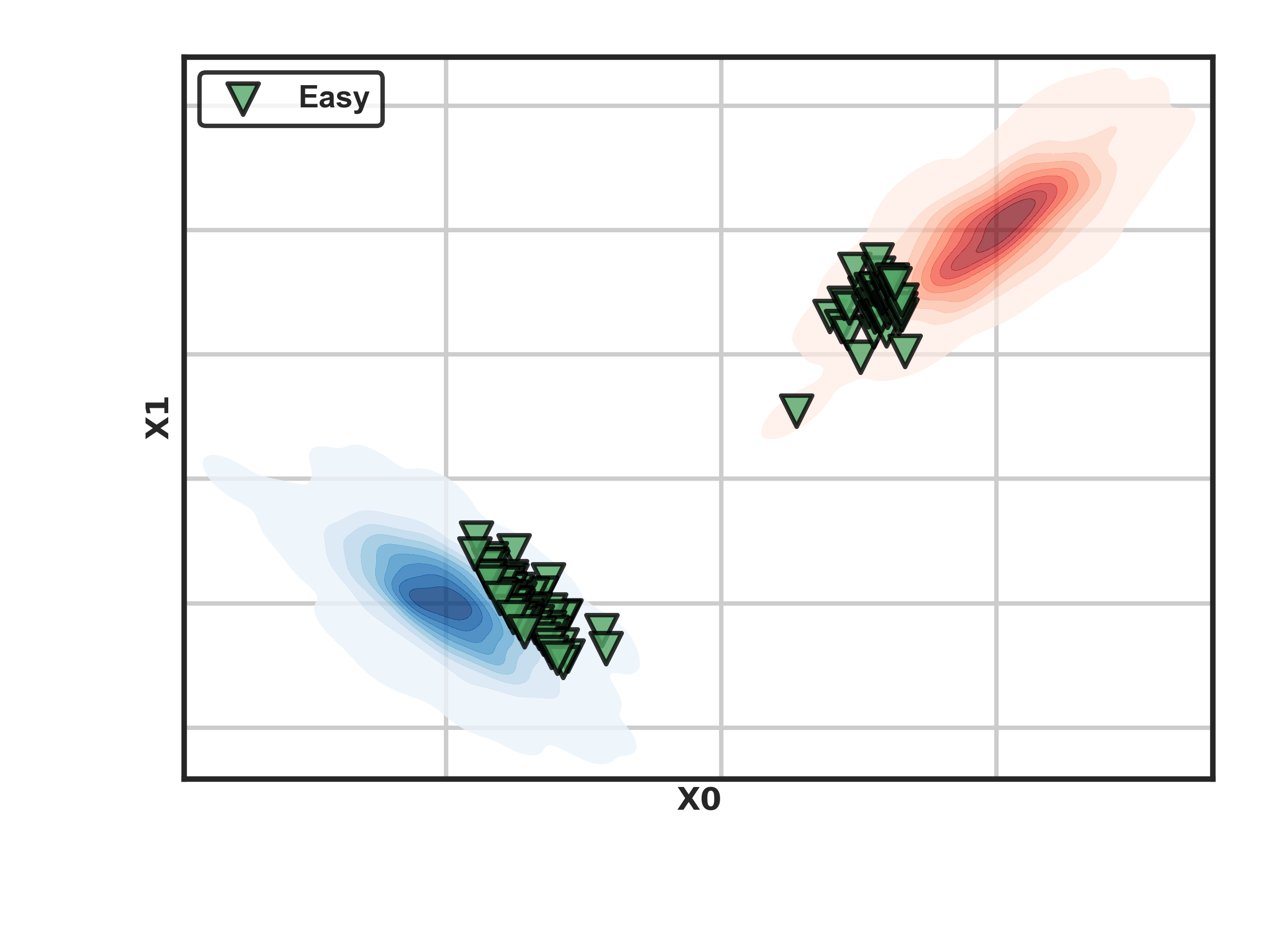}}
\\
\subfloat[\textsc{Hard ($\psi=0$)}]
{\includegraphics[width=0.32\textwidth]{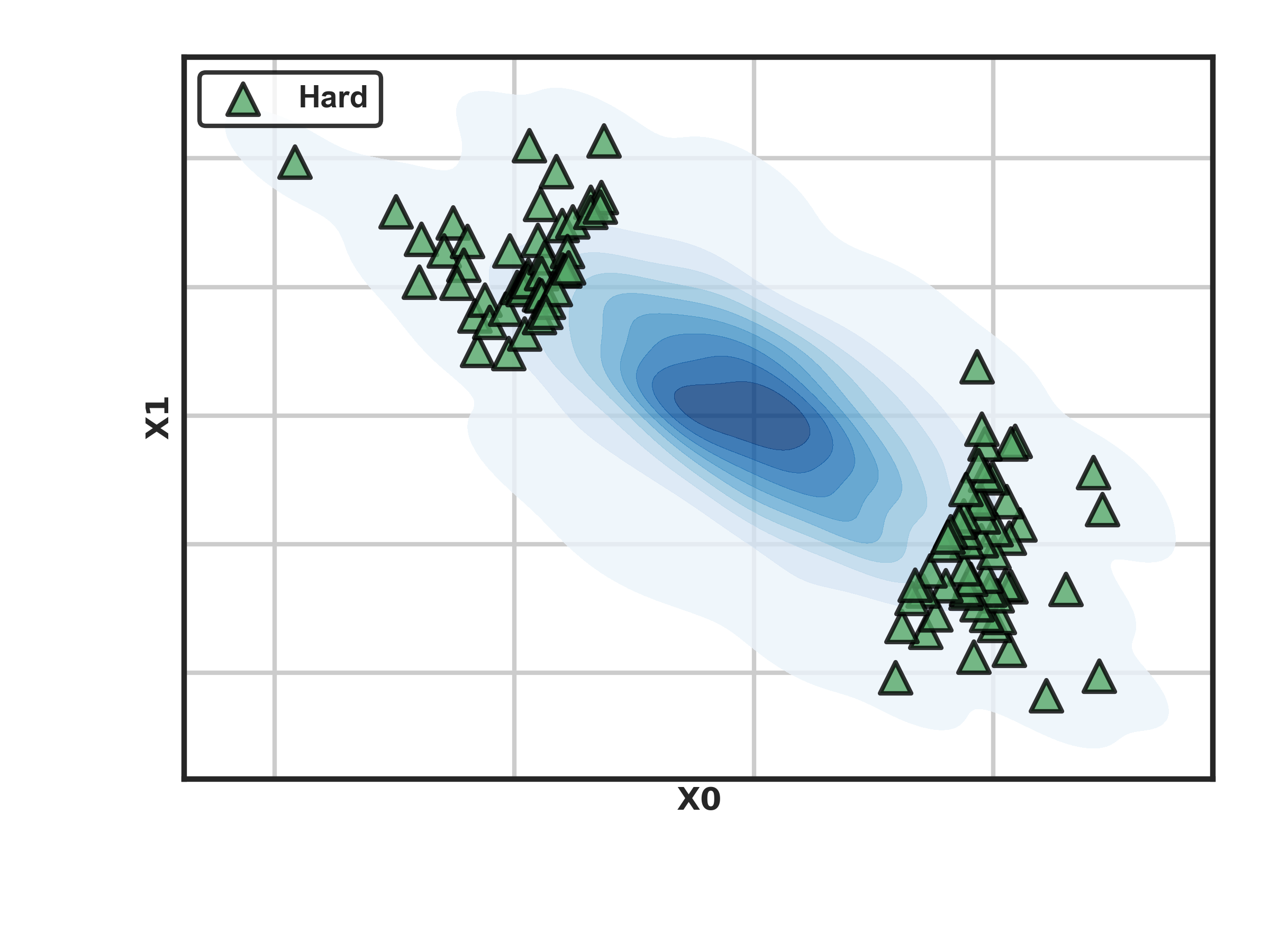}}
\subfloat[\textsc{Hard ($\psi=0.2$)}]
{\includegraphics[width=0.32\textwidth]{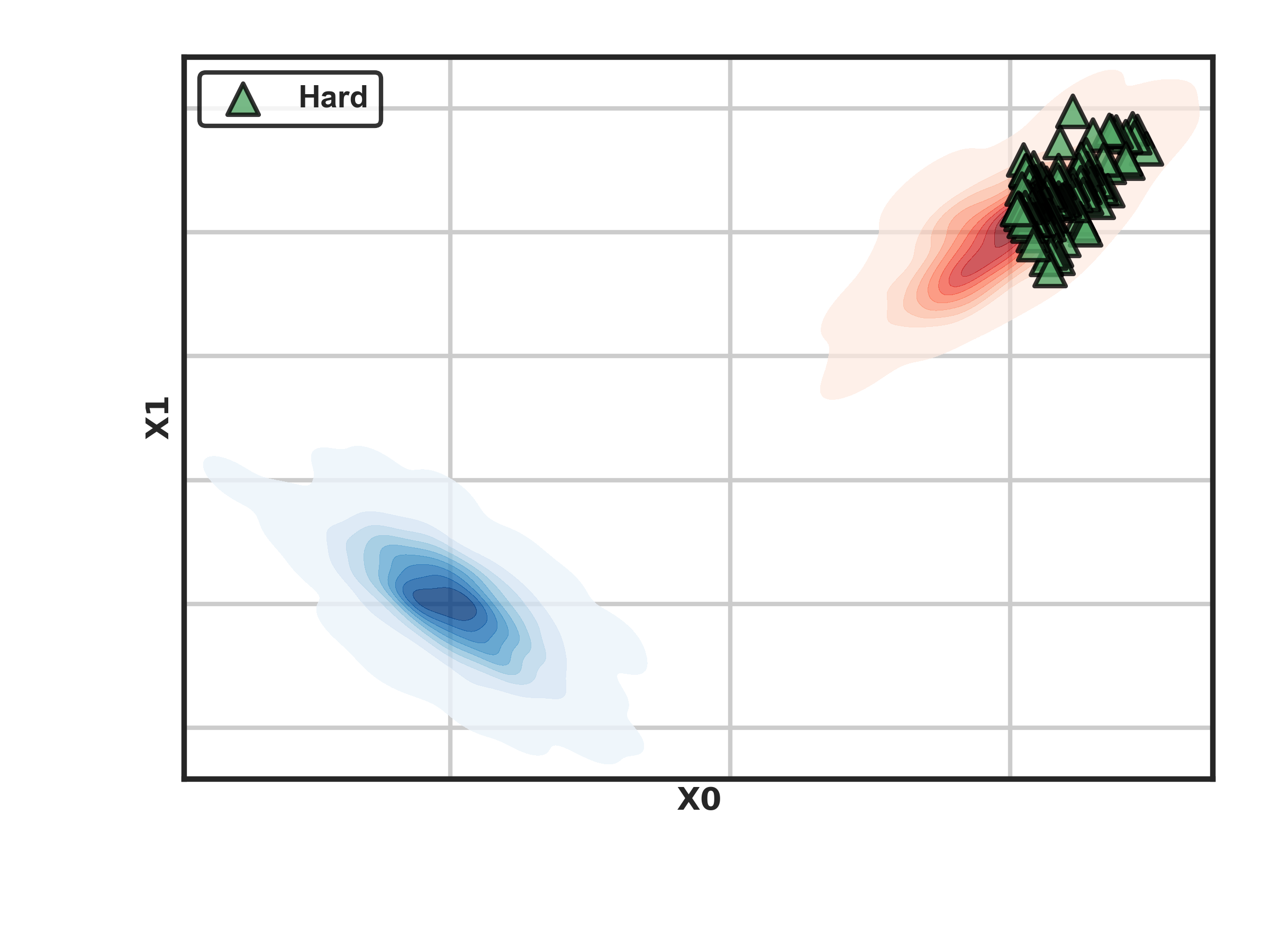}}
\subfloat[\textsc{Hard ($\psi=0.45$)}]
{\includegraphics[width=0.32\textwidth]{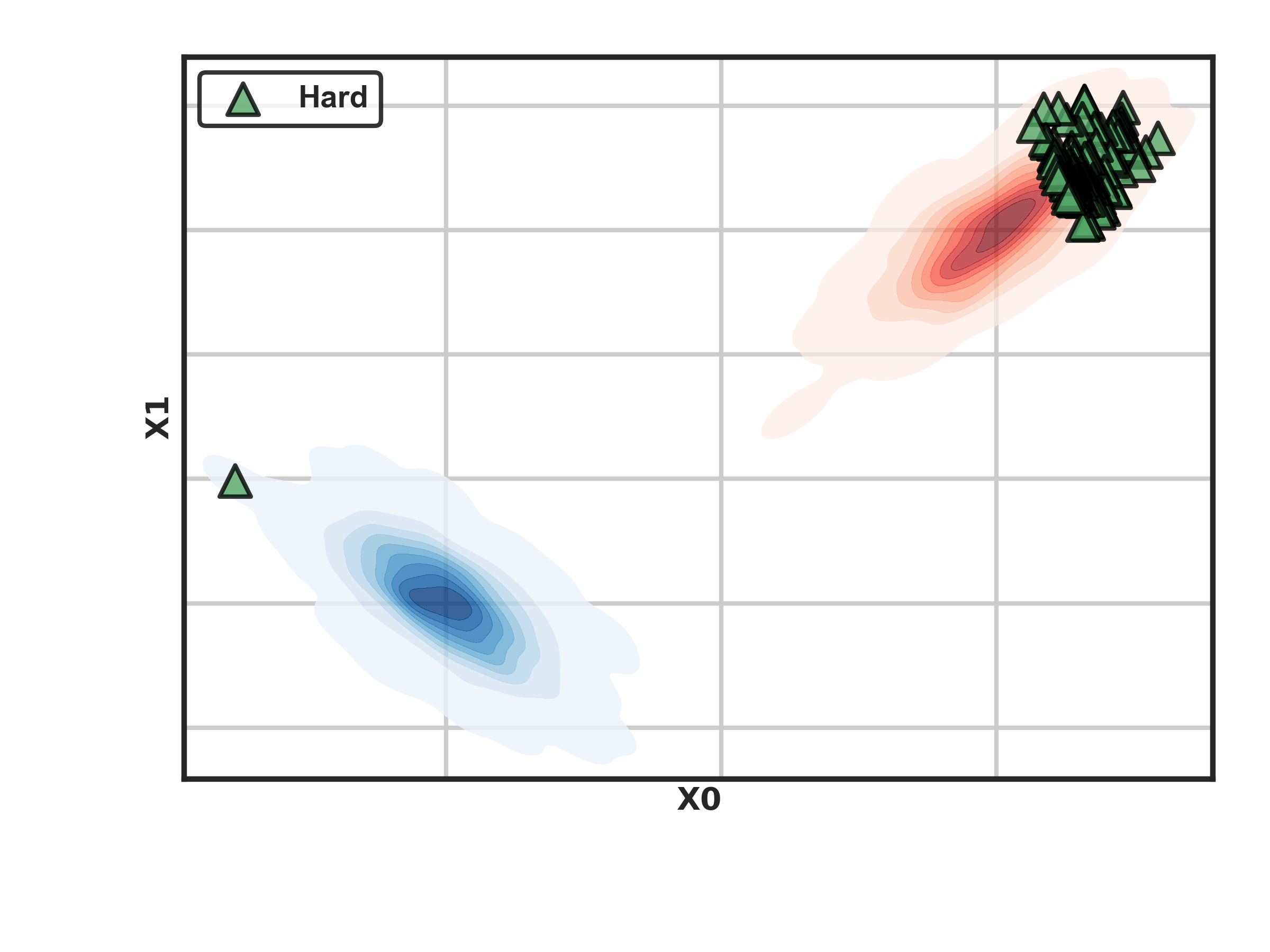}}
\\
\subfloat[\textsc{Gm Matching ($\psi=0$)}]
{\includegraphics[width=0.32\textwidth]{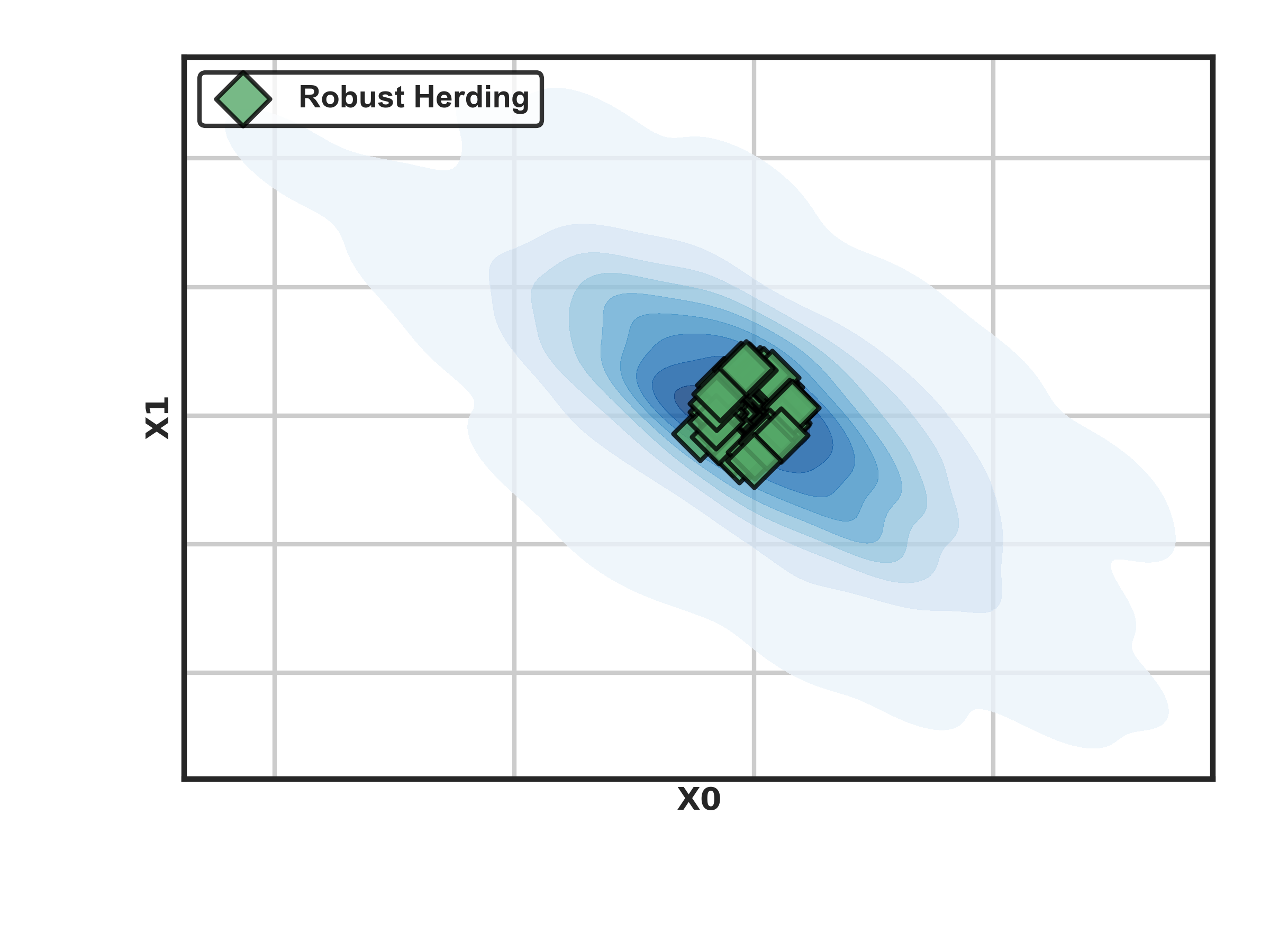}}
\subfloat[\textsc{Gm Matching ($\psi=0.2$)}]
{\includegraphics[width=0.32\textwidth]{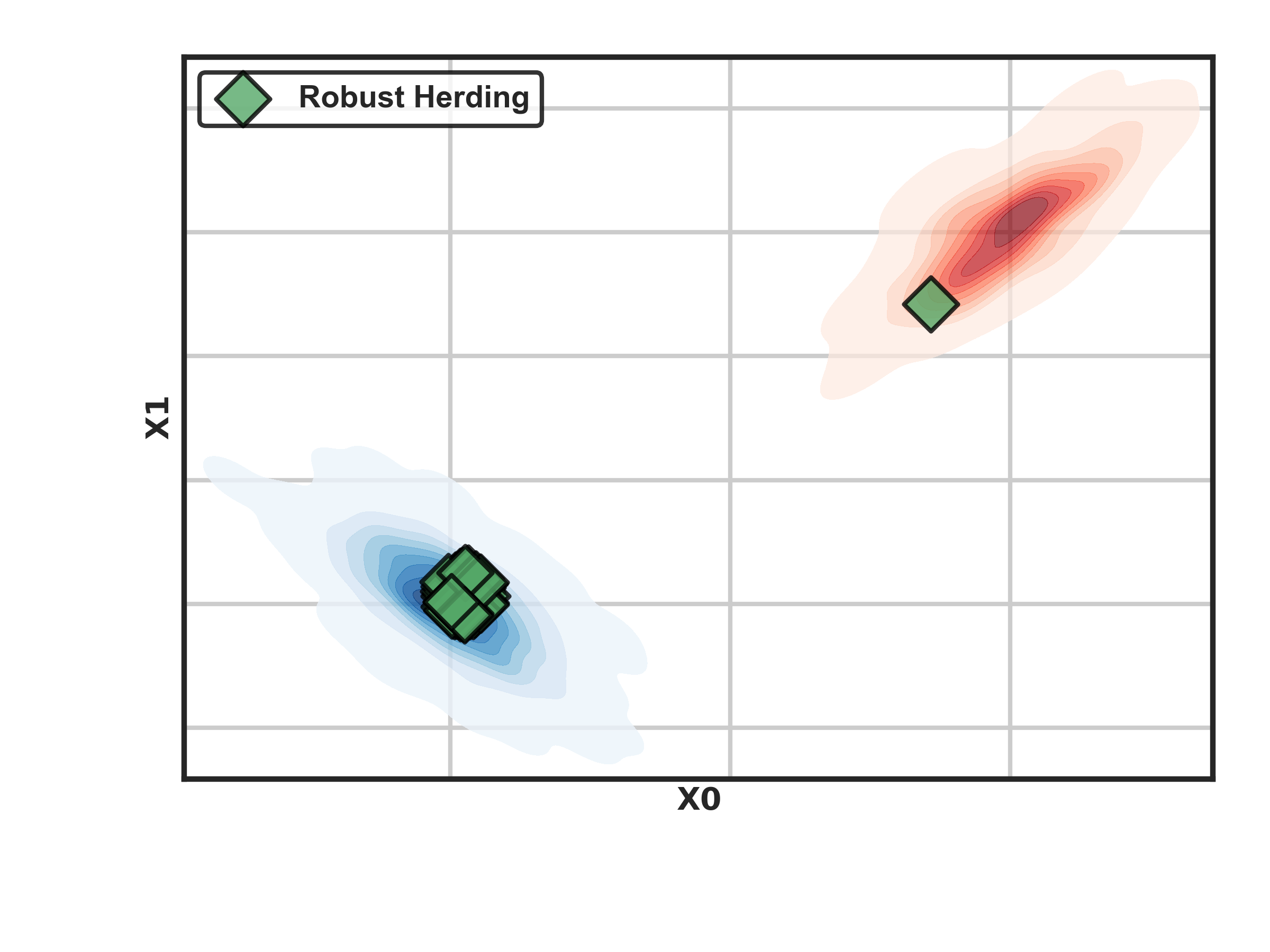}}
\subfloat[\textsc{Gm Matching ($\psi=0.45$)}]
{\includegraphics[width=0.32\textwidth]{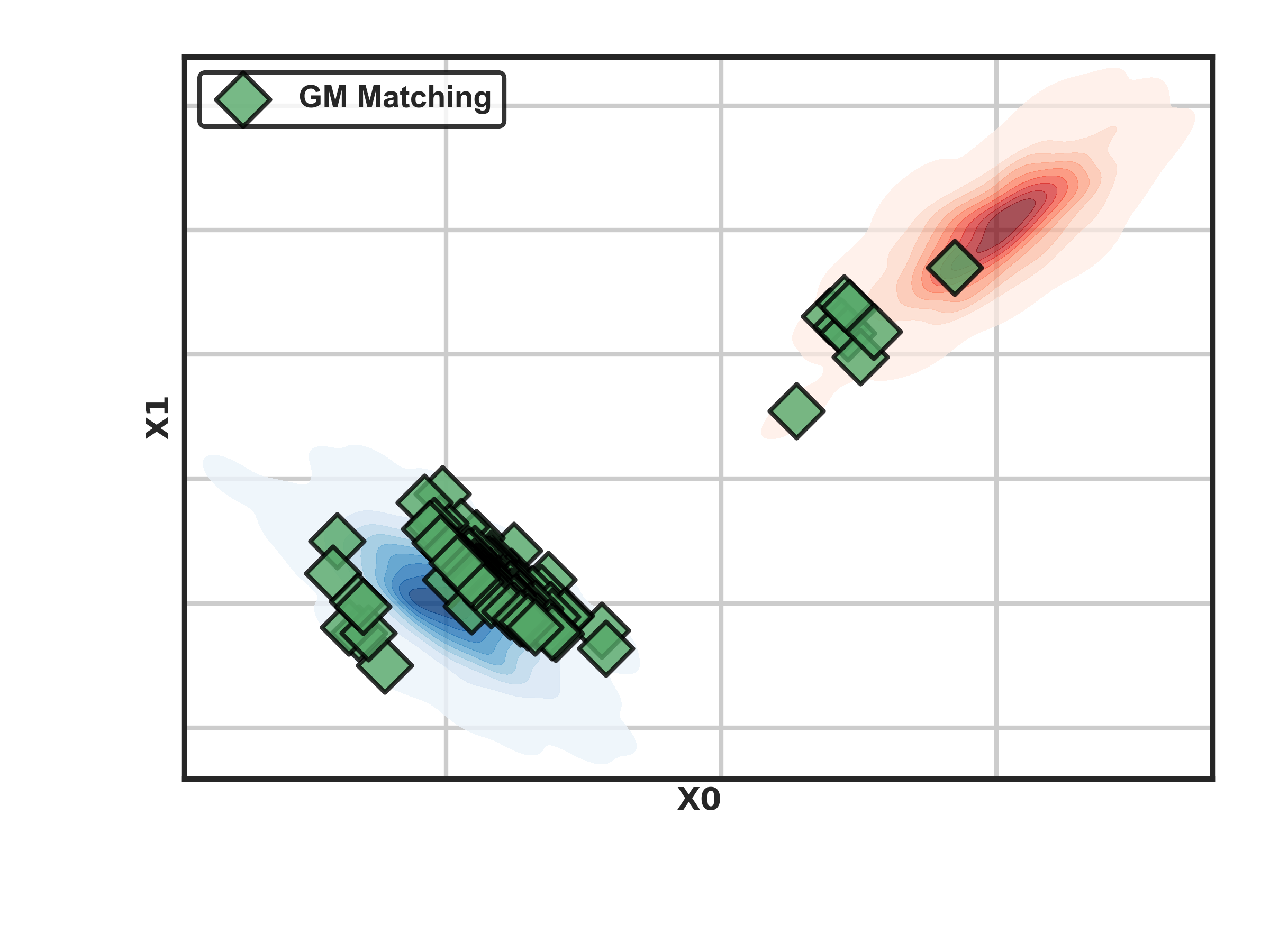}}
\\
\caption{\footnotesize {\bf Toy Example:} {\bf 0/20/45\% of the samples are corrupted} i.e. drawn from an adversary chosen distribution (red). We compare several baselines for choosing 10\% samples: (\textsc{Uniform}) random sampling, (\textsc{Easy}) selects of samples closest to the centroid. (\textsc{Hard}) Selection of samples farthest from the centroid. (\textsc{Moderate}) selects samples closest to the median distance from the centroid. (\textsc{Herding}) moment matching, (\textsc{GM Matching}) robust moment (GM) matching. Clearly $\gm$ Matching is significantly more robust and diverse than the other approaches even at such high corruption rates.}

\label{fig:toy-trade-off}
\end{figure*}

{\bf Robustness vs Diversity :}   In the {\bf ideal scenario} (i.e. in absence of any corruption), hard examples are known to contribute the most in downstream generalization performance~\citep{katharopoulos2018not, joshi2009multi, huang2010active, balcan2007margin} as they often capture most of the usable information in the dataset~\citep{xu2020theory}. On the other hand, in {\bf realistic noisy scenarios} involving outliers, this strategy often fails since the noisy examples are  wrongly deemed informative for training~\citep{zhang2018generalized, park2024robust}. Pruning methods specifically designed for such noisy scenarios thus propose to retain the most representative (easy) samples~\citep{pleiss2020identifying, jiang2018mentornet, har2006maximum, shah2020choosing, shen2019learning}. However, by only choosing samples far from the decision boundary, these methods ignore the more informative uncorrupted less prototypical samples. This can often result in sub-optimal downstream performance and in fact can also lead to degenerate solutions due to a covariance-shift problem~\citep{sugiyama2012machine}; giving rise to a {\em robustness vs diversity trade off}~\citep{xia2022moderate}.
This issue restricts the applicability of these methods, as realistic scenarios often deviate from expected conditions, making it challenging or impractical to adjust the criteria and methods accordingly. To go beyond these limitations, in this work, we consider the setting where a fraction of the samples can be {\bf arbitrarily corrupted}. 
\begin{definition}[{\bf $\alpha$-corrupted generation process}]
\label{def:corruption_model}
    Given $\psi \in [0, \frac{1}{2})$ and a set of observations from the original distribution of interest, an adversary is allowed to \textbf{inspect} all the samples and \textbf{arbitrarily} perturb up to $\psi$ fraction of them. We refer to a set of samples $\gD = \gD_{\gG} \cup \gD_{\gB}$ generated through this process as $\alpha$-corrupted where, $\alpha:=|\gD_{\gB}|/ |\gD_{\gG}|  = \frac{\psi}{1 - \psi}< 1$. $\gD_{\gB}$ and $\gD_{\gG}$ denote the sets of corrupt and clean samples.
\end{definition}
Given such an $\alpha$-corrupted~\footnote{
This strong corruption model is also referred to as the Gross Contamination Framework~\citep{diakonikolas2019recent}, that generalizes the popular Huber Contamination Model~\citep{huber1992robust} and Byzantine Corruption Framework~\citep{lamport1982byzantine}.} set of observations $\gD = \gD_{\gG} \cup \gD_{\gB} = \{(\vx_i, y_i) \in \sR^d \times \gY\}_{i=1}^N$, the goal of {\bf robust data pruning} is thus to judiciously select a subset $\gD_\gS \subseteq \gD$; that {\em encapsulates the comprehensive statistical characteristics of the underlying clean (uncorrupted) distribution induced by subset $\gD_{\gG}$ without any a-priori knowledge about the corrupted samples}. Note that, allowing the noisy samples to be {\bf arbitrarily corrupted} enables us to generalize many important robustness scenarios; including \textbf{corrupt features, label noise} and {\bf adversarial attacks}.

In response, we develop a selection strategy that aims to find subset such that the discrepancy between the mean of subset and the geometric median ($\gm$)(\cref{def:geo_med}) of the (potentially noisy) dataset is minimized. We call this algorithm Geometric Median Matching and describe it in~\cref{sec:gm_match}. We show that $\gm$ Matching is guaranteed to converge to a neighborhood of the underlying true (uncorrupted) mean of the (corrupted) dataset even when $1/2$ fraction of the samples are arbitrarily corrupted. 

Theoretically, we show that, $\gm$ Matching converges to a neighborhood of original underlying mean, at an impressive $\gO(1/k)$ rate while being robust.

\begin{table*}
    \footnotesize
    \centering
    \begin{tabular}{lccccc}
        \toprule
        & \multicolumn{2}{c}{\bf ResNet-50→ VGG-16} 
        & \multicolumn{2}{c}{\bf ResNet-50→ ShuffleNet} 
        & \\
        \cmidrule(r){2-3} 
        \cmidrule(r){4-5} 
        {\bf Method / Ratio } 
        & 20\% 
        & 30\% 
        & 20\% 
        & 30\% 
        & {\bf Mean $\uparrow$} 
        \\
        \midrule
        \multicolumn{6}{c}{\bf No Corruption} 
        \\
        \midrule 
        Random 
        & 29.63$\pm$0.43 
        & 35.38$\pm$0.83 
        & 32.40$\pm$1.06 
        & 39.13$\pm$0.81 
        & 34.96
        \\
        Herding 
        & 31.05$\pm$0.22 
        & 36.27$\pm$0.57 
        & 33.10$\pm$0.39 
        & 38.65$\pm$0.22 
        & 35.06
        \\
        Forgetting 
        & 27.53$\pm$0.36 
        & 35.61$\pm$0.39 
        & 27.82$\pm$0.56 
        & 36.26$\pm$0.51 
        & 32.35
        \\
        GraNd-score 
        & 29.93$\pm$0.95 
        & 35.61$\pm$0.39 
        & 29.56$\pm$0.46 
        & 37.40$\pm$0.38 
        & 33.34
        \\
        EL2N-score 
        & 26.47$\pm$0.31 
        & 33.19$\pm$0.51 
        & 28.18$\pm$0.27 
        & 35.81$\pm$0.29 
        & 31.13
        \\
        Optimization-based 
        & 25.92$\pm$0.64 
        & 34.82$\pm$1.29 
        & 31.37$\pm$1.14 
        & 38.22$\pm$0.78 
        & 32.55
        \\
        Self-sup.-selection 
        & 25.16$\pm$1.10 
        & 33.30$\pm$0.94 
        & 29.47$\pm$0.56 
        & 36.68$\pm$0.36 
        & 31.45
        \\
        Moderate-DS 
        & 31.45$\pm$0.32 
        & 37.89$\pm$0.36 
        & 33.32$\pm$0.41 
        & 39.68$\pm$0.34
        & 35.62
        \\
        {\bf $\gm$ Matching}
        & {\bf 35.86$\pm$0.41}
        & {\bf 40.56$\pm$0.22}
        & {\bf 35.51$\pm$0.32}
        & {\bf 40.30$\pm$0.58}
        & {\bf 38.47}
        \\
        \midrule
        \multicolumn{6}{c}{\bf 20\% Label Corruption} 
        \\
        \midrule 
        Random 
        & 23.29$\pm$1.12 
        & 28.18$\pm$1.84 
        & 25.08$\pm$1.32 
        & 31.44$\pm$1.21 
        & 27.00
        \\
        Herding 
        & 23.99$\pm$0.36
        & 28.57$\pm$0.40
        & 26.25$\pm$0.47 
        & 30.73$\pm$0.28
        & 27.39
        \\
        Forgetting 
        & 14.52$\pm$0.66 
        & 21.75$\pm$0.23 
        & 15.70$\pm$0.29 
        & 22.31$\pm$0.35 
        & 18.57
        \\
        GraNd-score 
        & 22.44$\pm$0.46 
        & 27.95$\pm$0.29 
        & 23.64$\pm$0.10 
        & 30.85$\pm$0.21 
        & 26.22
        \\
        EL2N-score 
        & 15.15$\pm$1.25 
        & 23.36$\pm$0.30 
        & 18.01$\pm$0.44 
        & 24.68$\pm$0.34 
        & 20.30
        \\
        Optimization-based 
        & 22.93$\pm$0.58 
        & 24.92$\pm$2.50 
        & 25.82$\pm$1.70 
        & 30.19$\pm$0.48 
        & 25.97
        \\
        Self-sup.-selection 
        & 18.39$\pm$1.30 
        & 25.77$\pm$0.87 
        & 22.87$\pm$0.54 
        & 29.80$\pm$0.36 
        & 24.21
        \\
        Moderate-DS 
        & 23.68$\pm$0.19 
        & 28.93$\pm$0.19 
        & 28.82$\pm$0.33 
        & 32.39$\pm$0.21 
        & 28.46
        \\
        {\bf $\gm$ Matching}
        & {\bf 28.77$\pm$0.77}
        & {\bf 34.87$\pm$0.23}
        & {\bf 32.05$\pm$0.93}
        & {\bf 37.43$\pm$0.25}
        & {\bf 33.28}
        \\
        \midrule
        \multicolumn{6}{c}{\bf 20\% Feature Corruption} 
        \\
        \midrule
        Random 
        & 26.33$\pm$0.88 
        & 31.57$\pm$1.31 
        & 29.15$\pm$0.83 
        & 34.72$\pm$1.00 
        & 30.44
        \\
        Herding 
        & 18.03$\pm$0.33 
        & 25.77$\pm$0.34 
        & 23.33$\pm$0.43 
        & 31.73$\pm$0.38 
        & 24.72
        \\
        Forgetting 
        & 19.41$\pm$0.57 
        & 28.35$\pm$0.16 
        & 18.44$\pm$0.57 
        & 31.09$\pm$0.61 
        & 24.32
        \\
        GraNd-score 
        & 23.59$\pm$0.19 
        & 30.69$\pm$0.13 
        & 23.15$\pm$0.56 
        & 31.58$\pm$0.95 
        & 27.25
        \\
        EL2N-score 
        & 24.60$\pm$0.81 
        & 31.49$\pm$0.33 
        & 26.62$\pm$0.34 
        & 33.91$\pm$0.56 
        & 29.16
        \\
        Optimization-based 
        & 25.12$\pm$0.34 
        & 30.52$\pm$0.89 
        & 28.87$\pm$1.25 
        & 34.08$\pm$1.92 
        & 29.65
        \\
        Self-sup.-selection 
        & 26.33$\pm$0.21 
        & 33.23$\pm$0.26 
        & 26.48$\pm$0.37 
        & 33.54$\pm$0.46 
        & 29.90
        \\
        Moderate-DS 
        & 29.65$\pm$0.68 
        & 35.89$\pm$0.53 
        & 32.30$\pm$0.38
        & 38.66$\pm$0.29 
        & 34.13
        \\
        $\gm$ Matching
        & {\bf 33.45$\pm$1.02}
        & {\bf 39.46$\pm$0.44}
        & {\bf 35.14$\pm$0.21}
        & {\bf 39.89$\pm$0.98}
        & {\bf 36.99}
        \\
        \midrule
        \multicolumn{6}{c}{\bf PGD Attack} 
        \\
        \midrule
        Random 
        & 26.12$\pm$1.09 
        & 31.98$\pm$0.78 
        & 28.28$\pm$0.90 
        & 34.59$\pm$1.18 
        & 30.24
        \\
        Herding 
        & 26.76$\pm$0.59 
        & 32.56$\pm$0.35 
        & 28.87$\pm$0.48 
        & 35.43$\pm$0.22 
        & 30.91
        \\
        Forgetting 
        & 24.55$\pm$0.57 
        & 31.83$\pm$0.36 
        & 23.32$\pm$0.37 
        & 31.82$\pm$0.15 
        & 27.88
        \\
        GraNd-score 
        & 25.19$\pm$0.33 
        & 31.46$\pm$0.54 
        & 26.03$\pm$0.66 
        & 33.22$\pm$0.24 
        & 28.98
        \\
        EL2N-score 
        & 21.73$\pm$0.47 
        & 27.66$\pm$0.32 
        & 22.66$\pm$0.35 
        & 29.89$\pm$0.64 
        & 25.49
        \\
        Optimization-based 
        & 26.02$\pm$0.36 
        & 31.64$\pm$1.75 
        & 27.93$\pm$0.47 
        & 34.82$\pm$0.96 
        & 30.10
        \\
        Self-sup.-selection 
        & 22.36$\pm$0.30 
        & 28.56$\pm$0.50 
        & 25.35$\pm$0.27 
        & 32.57$\pm$0.13 
        & 27.21
        \\
        Moderate-DS 
        & 27.24$\pm$0.36 
        & 32.90$\pm$0.31 
        & 29.06$\pm$0.28 
        & 35.89$\pm$0.53 
        & 31.27
        \\
        {\bf $\gm$ Matching}
        & {\bf 27.96$\pm$1.60}
        & {\bf 35.76$\pm$0.82}
        & {\bf 34.11$\pm$0.65}
        & {\bf 40.91$\pm$0.84}
        & {\bf 34.69}
        \\
        \bottomrule
    \end{tabular}
    \caption{{\bf Tiny ImageNet} : A ResNet-50 proxy (pretrained on TinyImageNet) is used to find important samples from Tiny-ImageNet; which is then used to train a VGGNet-16 and ShuffleNet. We repeat the experiment across multiple corruption settings - clean; 20\% Feature / Label Corruption and PGD attack when 20\% and 30\% samples are selected.}
    \label{tab:VGG-Shuffle}
\end{table*}

\section{Geometric Median (\textsc{GM}) Matching}
\label{sec:gm_match}
We make a {\bf key} observation that in presence of arbitrary corruption, the class center itself can be arbitrarily shifted and in fact need not even lie in the convex hull of the underlying clean samples. That is to say, in the arbitrary corruption scenario, the entire notion of easy (robust) / hard based on heuristic falls apart.
% Under the gross contamination framework~(\cref{def:corruption_model}), the vulnerability of pruning algorithms can be attributed to the centroid estimation. Motivated by this {\bf key observation} in this paper we propose $\gmm$ -- a data pruning strategy that remains provably effective even when up to $1/2$ fraction of the samples are arbitrary corrupted without any distributional assumption on the data. 
\begin{algorithm}
\SetAlgoLined
{\bf Initialize :} A finite collection of $\alpha$ corrupted (\cref{def:corruption_model}) observations $\gD$ defined over Hilbert space $\gH \in \sR^d$, equipped with norm $\|\cdot\|$ and inner $\langle\cdot\rangle$ operators; initial weight vector $\vtheta_0 \in \gH$. 
\\
{\bf Robust Mean Estimation:} $\Mu^\gm =\argmin_{\vz\in \gH}\sum_{\vx_i \in \gD}\|\vz - \vx_i\|$
\\
$\gD_\gS \leftarrow \emptyset$
\\
\For{iterations  t = 0, 1, \dots , k-1}
{
    $\vx_{t+1} := \argmax_{\vx \in \gD} \; \langle\vtheta_t,\vx\rangle$
    \\
    $\vtheta_{t+1} := \vtheta_t + \Mu^\gm_\epsilon - \vx_{t+1}$
    \\
    $\gD_\gS := \gD_\gS \cup \vx_{t+1}$
    \\
    $\gD: = \gD \setminus \vx_{t+1}$
}

{\bf return:} $\gD_\gS$
\label{algo:gmm}
\caption{\bf\textsc{Geometric Median Matching}}
\end{algorithm}
We exploit the breakdown and translation invariance property of Geometric Median ($\gm$) (~\cref{def:geo_med}) -- a well studied spatial estimator, known for several nice properties like {\bf rotation and translation invariance} and \textbf{optimal breakdown point of 1/2 under gross corruption}~\citep{minsker2015geometric, kemperman1987median}. to perform subset selection while being resilient to arbitrary corruption.

\begin{definition}[{\textbf{Geometric Median}}]
\label{def:geo_med}
    Given a finite collection of observations $\{\vx_1, \vx_2, \dots \vx_n\}$ defined over Hilbert space $\gH \in \sR^d$, equipped with norm $\|\cdot\|$ and inner $\langle\cdot\rangle$ operators,
    the geometric median(or Fermat-Weber point)~\citep{weber1929alfred} is defined as:
    \begin{equation}
    \label{eq:gm}
        \vx_* =\argmin_{\vz\in \gH}\bigg[ \rho(\vz):= \sum_{i=1}^{n}\bigg\|\vz - \vx_i\bigg\|
        \bigg] 
    \end{equation}
\end{definition}

Computing the $\gm$ exactly is known to be hard and no linear time algorithm exists~\citep{bajaj1988algebraic}. Consequently, it is necessary to rely on approximation methods to estimate the geometric median~\citep{weiszfeld1937point, vardi2000multivariate, cohen2016geometric}. We call a point $\vx_\epsilon \in \gH$ an {\bf $\epsilon$-accurate geometric median} if:
\begin{equation}
    \sum_{i=1}^{n}\bigg\|\vx_\epsilon - \vx_i\bigg\| \leq (1 + \epsilon)\sum_{i=1}^{n}\bigg\|\vx_\ast - \vx_i\bigg\|
\label{eq:gm_epsilon}
\end{equation}

Given a random batch of samples $\gD$ from an $\alpha$ corrupted dataset (~\cref{def:corruption_model}), access to an $\epsilon$ accurate $\gm(\cdot)$ oracle; we aim to find a $k$-subset of samples such that the mean of the selected subset approximately matches the geometric median of the training dataset, i.e. we aim solve:
\begin{equation}
    \argmin_{\gD_\gS \subseteq \gD, |\gD_\gS| = k}
    \bigg\| \Mu^\gm_\epsilon - \frac{1}{k} \sum_{\vx_i \in \gS}\vx_i\bigg\|^2
    \label{eq:gm_matching}
\end{equation}
where, $\Mu^\gm_\epsilon$ denotes the $\epsilon$-approximate $\gm$~\eqref{eq:gm_epsilon} of training dataset $\gD$. Intuitively, if such a subset can be found, the expected cost (e.g. loss / average gradient) over $\gS$ is guaranteed to be close (as characterized in~\cref{th:gm_match_cvg}) to the expected cost over the uncorrupted examples $\gD_\gG$; enabling robust data pruning. This is because, even in presence of grossly corrupted samples (unbounded), $\gm$ remains bounded w.r.t the underlying true mean~\citep{pmlr-v151-acharya22a, minsker2015geometric, cohen2016geometric, wu2020federated, chen2017distributed}.

Noting that, the proposed {\bf robust moment matching objective}~\eqref{eq:gm_matching} is an instance of the famous set function maximization problem -- known to be NP hard via a reduction from $k$-set cover~\citep{feige1998threshold, nemhauser1978analysis}; we adopt the iterative, herding~\citep{welling2009herding, chen2010super} style greedy approach to approximately solve~\eqref{eq:gm_matching}: 

\fcolorbox{black}{white}
{
    \begin{minipage}{0.95\linewidth}
        We start with a suitably chosen $\vtheta_0 \in \sR^d$; and repeatedly perform the following updates, adding one sample at a time, $k$ times:
        \begin{flalign}
            \vx_{t+1} &:= \argmax_{\vx \in \gD} \; \langle\vtheta_t, \vx\rangle
            \\
            \vtheta_{t+1} &:= \vtheta_t + \bigg( \Mu^\gm_\epsilon - \vx_{t+1} \bigg)
        \end{flalign}
    \end{minipage}
}

Clearly, $\gmm$ is performing greedy minimization of the error~\eqref{eq:gm_matching}. 
Intuitively, at each iteration, it accumulates the discrepancies between the $\gm$ of dataset and empirical mean of the chosen samples. By pointing towards the direction that reduces the accumulated error, $\mathbf{\theta}_t$ guides the algorithm to explore underrepresented regions of the feature space, thus {\bf promoting diversity}. Moreover, by matching {\bf $\gm$ (robust moment)}, the algorithm ensures that far-away points (outliers) suffer larger penalty -- consequently, encouraging $\gmm$ to choose more samples near the convex hull of uncorrupted points $\text{conv}\{\phi_\gB(\vx) | \vx \in \gD_\gG\}$.
Theoretically, we can establish the following convergence guarantee for $\gmm$ :
\begin{theorem}
    Suppose that, we are given, a set of $\alpha$-corrupted samples $\gD$ (\cref{def:corruption_model}), pretrained proxy model $\phi_\mB$, and an $\epsilon$ approx. $\gm(\cdot)$ oracle~\eqref{eq:gm}. Then, $\gmm$ guarantees that the mean of the selected $k$-subset $\bm{\mu}^\gS = \frac{1}{k} \sum_{\vx_i \in \gD_\gS}\vx_i$ converges to the neighborhood of $\bm{\mu}^\gG = \E_{\vx \in \gD_\gG}(\vx)$ at the rate $\gO(\frac{1}{k})$ such that:   
    \begin{flalign}
        \label{equation:bound-z0-simplified}
        \nonumber
        \E\bigg\|\bm{\mu}^\gS - \bm{\mu}^\gG\bigg \|^2\leq 
        &\frac{8|\gD_\gG|}{(|\gD_\gG|-|\gD_\gB|)^2}\sum_{\vx \in \gD_\gG}\E\bigg\|\vx -\Mu^\gG \bigg\|^2 
        \\
        &+\frac{2\epsilon^2}{(|\gD_\gG|-|\gD_\gB|)^2}
    \end{flalign}
\label{th:gm_match_cvg}
\end{theorem}
This result suggest that, even in presence of $\alpha$ corruption, the proposed algorithm $\gm$ Matching converges to a neighborhood of the true mean, where the neighborhood radius depends on two terms -- the first term depend on the variance of the uncorrupted samples and the second term depends on how accurately the $\gm$ is calculated. Consequently, we say that $\gm$ Matching achieves the optimal breakdown point.   

\subsection{Empirical Evidence}
To ensure reproducibility, our experimental setup is identical to~\citep{xia2022moderate}. 
Spanning across three popular image classification datasets - CIFAR10, CIFAR100 and Tiny-ImageNet - and popular deep nets including ResNet-18/50~\citep{he2016deep}, VGG-16~\citep{simonyan2014very}, ShuffleNet~\citep{ma2018shufflenet}, SENet~\citep{hu2018squeeze}, EfficientNet-B0\citep{tan2019efficientnet}; we compare $\gmm$ against several popular data pruning strategies as baselines:  (1) Random; (2) Herding \cite{welling2009herding}; (3) Forgetting \cite{toneva2018empirical}; (4) GraNd-score \cite{paul2021deep}; (5) EL2N-score \cite{paul2021deep}; (6) Optimization-based \cite{yang2022dataset}; (7) Self-sup.-selection \cite{sorscher2022beyond} and (8) Moderate~\citep{xia2022moderate}. We do not run these baselines for be these baselines are borrowed from~\citep{xia2020robust}. 

We consider three corruption scenarios: {\bf (1) Image Corruption} : a popular robustness setting, often encountered when training models on real-world data~\citep{hendrycks2019benchmarking, szegedy2013intriguing}. {(2)\bf Label Noise :} data in the wild always contains noisy annotations~\citep{li2022selective, patrini2017making, xia2020robust}. {\bf (3) Adversarial Attack :} Imperceptible but adversarial noise on natural examples e.g. PGD attack~\citep{madry2017towards} and GS Attacks~\citep{goodfellow2014explaining}. 

Overall, we improve over prior work almost in all settings, the gains are especially more profound in presence of corruption and at aggressive pruning rates. Thus, making GM Matching a strong baseline for future research in robust data pruning. Due to space constraints we defer the experiments and discussion to Appendix while providing a glimpse in~\cref{fig:toy-trade-off},~\cref{tab:VGG-Shuffle}.

\section{Conclusion}
\label{sec:conclusion}
In this work, we formalized and studied the problem of robust data pruning. We show that existing data pruning strategies suffer significant degradation in performance in presence of corruption. Orthogonal to existing works, we propose $\gmm$ where our goal is to find a $k$-subset from the noisy data such that the mean of the subset approximates the $\gm$ of the noisy dataset. We solve this meta problem using a herding style greedy approach. We theoretically justify our approach and empirically show its efficacy by comparing $\gmm$ against several popular benchmarks across multiple datasets. Our results indicate that $\gmm$ consistently outperforms existing pruning strategies in both clean and noisy settings. While in this work we have only explored greedy herding style approach, it is possible to investigate other combinatorial approaches to solve the meta problem. Further, while we only studied the problem under gross corruption framework, it remains open to improve the results by incorporating certain structural assumptions. 

\clearpage
\bibliographystyle{apalike}
\bibliography{
bibs/contrastive, 
bibs/coreset,
bibs/scaling,
bibs/robust_estimation,
bibs/robust_sgd,
bibs/LLM,
bibs/theory_geometric,
bibs/imp_sampling,
bibs/exp,
bibs/subset_selection
}
\clearpage
\onecolumn
\begin{center}
    \textbf{\Large Supplementary Material}\vspace{5mm}
\end{center}

\section{Experiments}
In this section, we outline our experimental setup, present our key empirical findings, and discuss deeper insights into the performance of $\gm$ Matching. 

\subsection{Experimental Setup}
{\bf A. Baselines: } 

To ensure reproducibility, our experimental setup is identical to~\citep{xia2022moderate}. We compare the proposed $\gm$ Matching selection strategy against the following popular data pruning strategies as baselines for comparison:  (1) Random; (2) Herding \cite{welling2009herding}; (3) Forgetting \cite{toneva2018empirical}; (4) GraNd-score \cite{paul2021deep}; (5) EL2N-score \cite{paul2021deep}; (6) Optimization-based \cite{yang2022dataset}; (7) Self-sup.-selection \cite{sorscher2022beyond} and (8) Moderate~\citep{xia2022moderate}. We do not run these baselines for be these baselines are borrowed from~\citep{xia2020robust}.

Additionally, for further ablations we compare $\gm$ Matching with many (natural) distance based geometric pruning strategies: ({\bf \textsc{Uniform}}) Random Sampling, ({\bf \textsc{Easy}}) Selection of samples closest to the centroid; ({\bf \textsc{Hard}}) Selection of samples farthest from the centroid; ({\bf \textsc{Moderate}})~\citep{xia2022moderate} Selection of samples closest to the median distance from the centroid; ({\bf \textsc{Herding}}) Moment Matching~\citep{chen2010super}, ({\bf \textsc{GM Matching}}) Robust Moment (GM) Matching~\eqref{eq:gm_matching}.

{\bf B. Datasets and Networks :}

We perform extensive experiments across three popular image classification datasets - CIFAR10, CIFAR100 and Tiny-ImageNet. Our experiments span popular deep nets including ResNet-18/50~\citep{he2016deep}, VGG-16~\citep{simonyan2014very}, ShuffleNet~\citep{ma2018shufflenet}, SENet~\citep{hu2018squeeze}, EfficientNet-B0\citep{tan2019efficientnet}. 

{\bf C. Implementation Details : } 

For the CIFAR-10/100 experiments, we utilize a batch size of 128 and employ SGD optimizer with a momentum of 0.9, weight decay of 5e-4, and an initial learning rate of 0.1. The learning rate is reduced by a factor of 5 after the 60th, 120th, and 160th epochs, with a total of 200 epochs. Data augmentation techniques include random cropping and random horizontal flipping. In the Tiny-ImageNet experiments, a batch size of 256 is used with an SGD optimizer, momentum of 0.9, weight decay of 1e-4, and an initial learning rate of 0.1. The learning rate is decreased by a factor of 10 after the 30th and 60th epochs, with a total of 90 epochs. Random horizontal flips are applied for data augmentation. Each experiment is repeated over 5 random seeds and the variances are noted.

{\bf D. Proxy Model :}

Needless to say, identifying sample importance is an ill-posed problem without some notion of similarity among the samples. Thus, it is common to assume access to a proxy encoder $\phi_\mB(\cdot): \sR^d \to \gH \in \sR^p$ that maps the features to a separable Hilbert space equiped with inner product. Intuitively, this simply means that the embedding space of the encoder fosters proximity of semantically similar examples, while enforcing the separation of dissimilar ones -- a property often satisfied by even off-the-shelf pretrained foundation models~\citep{hessel2021clipscore, sorscher2022beyond}. We perform experiments across multiple choices of such proxy encoder scenarios: In~\cref{tab:clean}-~\ref{tab:Adv-A}, we show results in the {\bf standard setting: } when the proxy model shares the same architecture as the model e.g. ResNet50 to be trained downstream, and is pretrained (supervised) on the clean target dataset e.g. TinyimageNet. However, we also experiment with {(a) \bf distribution shift: } proxy model pretrained on a different (distribution shifted) dataset e.g. ImageNet and used to sample from Mini ImageNet. {(b) Network Transfer: } Where, the proxy has a different network compared to the classifier. We describe these experiments in more detail in~\cref{sec:proxy_ablations}.

\begin{table*}
\footnotesize
\centering
\resizebox{\textwidth}{!}
{
    \begin{tabular}{lccccccc}
        \midrule
        \rowcolor{light-gray}
        \multicolumn{8}{c}{{\bf CIFAR-100}}
        \\
        \midrule
        {\textbf{Method / Ratio}} 
        & \textbf{20\%} 
        & \textbf{30\%} 
        & \textbf{40\%}
        & \textbf{60\%}
        & \textbf{80\%}
        & \textbf{100\%}
        & \textbf{Mean $\uparrow$}
        \\ 
        \midrule
        Random              
        & 50.26$\pm$3.24 
        & 53.61$\pm$2.73 
        & 64.32$\pm$1.77 
        & 71.03$\pm$0.75 
        & 74.12$\pm$0.56 
        & 78.14$\pm$0.55 
        & 62.67
        \\ 
        Herding               
        & 48.39$\pm$1.42 
        & 50.89$\pm$0.97 
        & 62.99$\pm$0.61 
        & 70.61$\pm$0.44 
        & 74.21$\pm$0.49 
        & 78.14$\pm$0.55 
        & 61.42
        \\ 
        Forgetting            
        & 35.57$\pm$1.40 
        & 49.83$\pm$0.91 
        & 59.65$\pm$2.50 
        & \textbf{73.34$\pm$0.39} 
        & \textbf{77.50$\pm$0.53} 
        & 78.14$\pm$0.55 
        & 59.18
        \\ 
        GraNd-score           
        & 42.65$\pm$1.39 
        & 53.14$\pm$1.28 
        & 60.52$\pm$0.79 
        & 69.70$\pm$0.68 
        & 74.67$\pm$0.79 
        & 78.14$\pm$0.55
        & 60.14
        \\ 
        EL2N-score            
        & 27.32$\pm$1.16 
        & 41.98$\pm$0.54 
        & 50.47$\pm$1.20 
        & 69.23$\pm$1.00 
        & 75.96$\pm$0.88 
        & 78.14$\pm$0.55
        & 52.99
        \\ 
        Optimization-based    
        & 42.16$\pm$3.30 
        & 53.19$\pm$2.14 
        & 58.93$\pm$0.98 
        & 68.93$\pm$0.70 
        & 75.62$\pm$0.33 
        & 78.14$\pm$0.55 
        & 59.77
        \\
        Self-sup.-selection   
        & 44.45$\pm$2.51
        & 54.63$\pm$2.10 
        & 62.91$\pm$1.20 
        & 70.70$\pm$0.82 
        & 75.29$\pm$0.45 
        & 78.14$\pm$0.55 
        & 61.60
        \\ 
        Moderate-DS
        & 51.83$\pm$0.52
        & 57.79$\pm$1.61
        & 64.92$\pm$0.93
        & 71.87$\pm$0.91 
        & 75.44$\pm$0.40 
        & 78.14$\pm$0.55 
        & 64.37
        \\ 
        {\bf $\gm$ Matching}  
        & {\bf 55.93$\pm$ 0.48}
        & {\bf 63.08$\pm$ 0.57}
        & {\bf 66.59$\pm$ 1.18} 
        & 70.82$\pm$ 0.59
        & 74.63$\pm$ 0.86
        & 78.14$\pm$ 0.55  
        & {\bf 66.01}
        \\
        \midrule
        \rowcolor{light-gray}
        \multicolumn{8}{c}{{\bf Tiny ImageNet}}
        \\
        \midrule
        Random              
        & 24.02$\pm$0.41 
        & 29.79$\pm$0.27 
        & 34.41$\pm$0.46 
        & 40.96$\pm$0.47 
        & 45.74$\pm$0.61 
        & 49.36$\pm$0.25 
        & 34.98 
        \\ 
        Herding               
        & 24.09$\pm$0.45 
        & 29.39$\pm$0.53 
        & 34.13$\pm$0.37 
        & 40.86$\pm$0.61 
        & 45.45$\pm$0.33 
        & 49.36$\pm$0.25 
        & 34.78 
        \\ 
        Forgetting            
        & 22.37$\pm$0.71 
        & 28.67$\pm$0.54 
        & 33.64$\pm$0.32 
        & 41.14$\pm$0.43 
        & \textbf{46.77$\pm$0.31} 
        & 49.36$\pm$0.25 
        & 34.52
        \\ 
        GraNd-score           
        & 23.56$\pm$0.52 
        & 29.66$\pm$0.37 
        & 34.33$\pm$0.50 
        & 40.77$\pm$0.42 
        & 45.96$\pm$0.56 
        & 49.36$\pm$0.25 
        & 34.86 
        \\ 
        EL2N-score            
        & 19.74$\pm$0.26 
        & 26.58$\pm$0.40 
        & 31.93$\pm$0.28 
        & 39.12$\pm$0.46 
        & 45.32$\pm$0.27 
        & 49.36$\pm$0.25 
        & 32.54
        \\ 
        Optimization-based    
        & 13.88$\pm$2.17 
        & 23.75$\pm$1.62 
        & 29.77$\pm$0.94 
        & 37.05$\pm$2.81 
        & 43.76$\pm$1.50 
        & 49.36$\pm$0.25 
        & 29.64
        \\
        Self-sup.-selection   
        & 20.89$\pm$0.42 
        & 27.66$\pm$0.50 
        & 32.50$\pm$0.30 
        & 39.64$\pm$0.39 
        & 44.94$\pm$0.34 
        & 49.36$\pm$0.25 
        & 33.13
        \\ 
        Moderate-DS
        & 25.29$\pm$0.38
        & 30.57$\pm$0.20
        & 34.81$\pm$0.51
        & 41.45$\pm$0.44
        & 46.06$\pm$0.33 
        & 49.36$\pm$0.25 
        & 35.64
        \\ 
        {\bf $\gm$ Matching}  
        & {\bf 27.88$\pm$0.19}
        & {\bf 33.15$\pm$0.26}
        & {\bf 36.92$\pm$0.40}
        & {\bf 42.48$\pm$0.12}
        & 46.75$\pm$0.51
        & 49.36$\pm$0.25 
        & {\bf 37.44}
        \\
        \bottomrule
    \end{tabular}
}
    \caption{
    \footnotesize
    {\bf No Corruption :} Comparing (Test Accuracy) pruning algorithms on CIFAR-100 and Tiny-ImageNet in the uncorrupted setting. ResNet-50 is used both as proxy and for downstream classification.}
    \label{tab:clean}
\end{table*}

\subsection{Ideal (No Corruption) Scenario}
Our first sets of experiments involve performing data pruning across selection ratio ranging from 20\% - 80\% in the uncorrupted setting. The corresponding results, presented in~\cref{tab:clean}, indicate that while $\gm$ Matching is developed with robustness scenarios in mind, it outperforms the existing strong baselines even in the clean setting. Overall, on both CIFAR-100 and Tiny ImageNet $\gm$ Matching improves over the prior methods > 2\% on an average. In particular, we note that $\gm$ Matching enjoys larger gains in the low data selection regime, while staying competitive at low pruning rates.

\begin{table*}
\footnotesize
\centering
\resizebox{\textwidth}{!}
{
    \begin{tabular}{lccccccc}
        \toprule
        \rowcolor{light-gray}
        \multicolumn{8}{c}{{\bf CIFAR-100}}
        \\
        \midrule
        {\textbf{Method / Selection ratio}} 
        & \textbf{20\%} 
        & \textbf{30\%} 
        & \textbf{40\%}
        & \textbf{60\%}
        & \textbf{80\%}
        & \textbf{100\%}
        & \textbf{Mean $\uparrow$}
        \\ 
        \midrule
        \multicolumn{8}{c}{{\bf No Corruption}}
        \\
        \midrule
        Random              
        & 50.26$\pm$3.24 
        & 53.61$\pm$2.73 
        & 64.32$\pm$1.77 
        & 71.03$\pm$0.75 
        & 74.12$\pm$0.56 
        & 78.14$\pm$0.55 
        & 62.67
        \\ 
        Herding               
        & 48.39$\pm$1.42 
        & 50.89$\pm$0.97 
        & 62.99$\pm$0.61 
        & 70.61$\pm$0.44 
        & 74.21$\pm$0.49 
        & 78.14$\pm$0.55 
        & 61.42
        \\ 
        Forgetting            
        & 35.57$\pm$1.40 
        & 49.83$\pm$0.91 
        & 59.65$\pm$2.50 
        & \textbf{73.34$\pm$0.39} 
        & \textbf{77.50$\pm$0.53} 
        & 78.14$\pm$0.55 
        & 59.18
        \\ 
        GraNd-score           
        & 42.65$\pm$1.39 
        & 53.14$\pm$1.28 
        & 60.52$\pm$0.79 
        & 69.70$\pm$0.68 
        & 74.67$\pm$0.79 
        & 78.14$\pm$0.55
        & 60.14
        \\ 
        EL2N-score            
        & 27.32$\pm$1.16 
        & 41.98$\pm$0.54 
        & 50.47$\pm$1.20 
        & 69.23$\pm$1.00 
        & 75.96$\pm$0.88 
        & 78.14$\pm$0.55
        & 52.99
        \\ 
        Optimization-based    
        & 42.16$\pm$3.30 
        & 53.19$\pm$2.14 
        & 58.93$\pm$0.98 
        & 68.93$\pm$0.70 
        & 75.62$\pm$0.33 
        & 78.14$\pm$0.55 
        & 59.77
        \\
        Self-sup.-selection   
        & 44.45$\pm$2.51
        & 54.63$\pm$2.10 
        & 62.91$\pm$1.20 
        & 70.70$\pm$0.82 
        & 75.29$\pm$0.45 
        & 78.14$\pm$0.55 
        & 61.60
        \\ 
        Moderate-DS
        & 51.83$\pm$0.52
        & 57.79$\pm$1.61
        & 64.92$\pm$0.93
        & 71.87$\pm$0.91 
        & 75.44$\pm$0.40 
        & 78.14$\pm$0.55 
        & 64.37
        \\ 
        {\bf $\gm$ Matching}  
        & {\bf 55.93$\pm$ 0.48}
        & {\bf 63.08$\pm$ 0.57}
        & {\bf 66.59$\pm$ 1.18} 
        & 70.82$\pm$ 0.59
        & 74.63$\pm$ 0.86
        & 78.14$\pm$ 0.55  
        & {\bf 66.01}
        \\
        \midrule
        \multicolumn{8}{c}{{\bf 5\% Feature Corruption}}
        \\
        \midrule
        Random              
        & 43.14$\pm$3.04 
        & 54.19$\pm$2.92 
        & 64.21$\pm$2.39 
        & 69.50$\pm$1.06 
        & 72.90$\pm$0.52 
        & 77.26$\pm$0.39 
        & 60.79
        \\ 
        Herding               
        & 42.50$\pm$1.27 
        & 53.88$\pm$3.07 
        & 60.54$\pm$0.94 
        & 69.15$\pm$0.55 
        & 73.47$\pm$0.89 
        & 77.26$\pm$0.39 
        & 59.81
        \\ 
        Forgetting            
        & 32.42$\pm$0.74 
        & 49.72$\pm$1.64 
        & 54.84$\pm$2.20 
        & 70.22$\pm$2.00 
        & 75.19$\pm$0.40 
        & 77.26$\pm$0.39 
        & 56.48
        \\ 
        GraNd-score           
        & 42.24$\pm$0.57 
        & 53.48$\pm$0.76 
        & 60.17$\pm$1.66 
        & 69.16$\pm$0.81 
        & 73.35$\pm$0.81 
        & 77.26$\pm$0.39 
        & 59.68
        \\ 
        EL2N-score            
        & 26.13$\pm$1.75 
        & 39.01$\pm$1.42 
        & 49.89$\pm$1.87 
        & 68.36$\pm$1.41 
        & 73.10$\pm$0.36 
        & 77.26$\pm$0.39 
        & 51.30
        \\ 
        Optimization-based    
        & 38.25$\pm$3.04 
        & 50.88$\pm$6.07 
        & 57.26$\pm$0.93 
        & 68.02$\pm$0.39 
        & 73.77$\pm$0.56 
        & 77.26$\pm$0.39 
        & 57.64
        \\
        Self-sup.-selection   
        & 44.24$\pm$0.48 
        & 55.99$\pm$1.21 
        & 61.03$\pm$0.59 
        & 69.96$\pm$1.07 
        & 74.56$\pm$1.17 
        & 77.26$\pm$0.39 
        & 61.16
        \\ 
        Moderate-DS
        & 46.78$\pm$1.90
        & 57.36$\pm$1.22 
        & 65.40$\pm$1.19
        & 71.46$\pm$0.19 
        & {\bf 75.64$\pm$0.61} 
        & 77.26$\pm$0.39
        & 63.33
        \\
        {\bf $\gm$ Matching}
        & {\bf 49.50$\pm$0.72}
        & {\bf 60.23$\pm$0.88}
        & {\bf 66.25$\pm$0.51}
        & {\bf 72.91$\pm$0.26}
        & 75.10$\pm$0.29
        & 77.26$\pm$0.39
        & {\bf 64.80}
        \\
        \midrule
        \multicolumn{8}{c}{{\bf 10\% Feature Corruption}}
        \\
        \midrule
        Random              
        & 43.27$\pm$3.01 
        & 53.94$\pm$2.78 
        & 62.17$\pm$1.29 
        & 68.41$\pm$1.21 
        & 73.50$\pm$0.73 
        & 76.50$\pm$0.63 
        & 60.26
        \\ 
        Herding               
        & 44.34$\pm$1.07 
        & 53.31$\pm$1.49 
        & 60.13$\pm$0.38 
        & 68.20$\pm$0.74 
        & 74.34$\pm$1.07 
        & 76.50$\pm$0.63 
        & 60.06
        \\ 
        Forgetting            
        & 30.43$\pm$0.70 
        & 47.50$\pm$1.43 
        & 53.16$\pm$0.44 
        & 70.36$\pm$0.82 
        & 75.11$\pm$0.71 
        & 76.50$\pm$0.63 
        & 55.31
        \\ 
        GraNd-score           
        & 36.36$\pm$1.06 
        & 52.26$\pm$0.66 
        & 60.22$\pm$1.39 
        & 68.96$\pm$0.62 
        & 72.78$\pm$0.51 
        & 76.50$\pm$0.63 
        & 58.12
        \\ 
        EL2N-score            
        & 21.75$\pm$1.56 
        & 30.80$\pm$2.23 
        & 41.06$\pm$1.23 
        & 64.82$\pm$1.48 
        & 73.47$\pm$1.30 
        & 76.50$\pm$0.63 
        & 46.38
        \\ 
        Optimization-based    
        & 37.22$\pm$0.39 
        & 48.92$\pm$1.38 
        & 56.88$\pm$1.48 
        & 67.33$\pm$2.15 
        & 72.94$\pm$1.90 
        & 76.50$\pm$0.63 
        & 56.68
        \\
        Self-sup.-selection   
        & 42.01$\pm$1.31 
        & 54.47$\pm$1.19 
        & 61.37$\pm$0.68 
        & 68.52$\pm$1.24 
        & 74.73$\pm$0.36 
        & 76.50$\pm$0.63 
        & 60.22
        \\ 
        Moderate-DS
        & 47.02$\pm$0.66 
        & 55.60$\pm$1.67 
        & 62.18$\pm$1.86
        & 71.83$\pm$0.78
        & {\bf 75.66$\pm$0.66}
        & 76.50$\pm$0.63
        & 62.46
        \\
        {\bf $\gm$ Matching}
        & {\bf 48.86$\pm$1.02}
        & {\bf 60.15$\pm$0.43}
        & {\bf 66.92$\pm$0.28}
        & {\bf 72.03$\pm$0.38}
        & 73.71$\pm$0.19
        & 76.50$\pm$0.63
        & {\bf 64.33}
        \\
        \midrule
        \multicolumn{8}{c}{{\bf 20\% Feature Corruption}}
        \\
        \midrule
        Random              
        & 40.99$\pm$1.46 
        & 50.38$\pm$1.39 
        & 57.24$\pm$0.65 
        & 65.21$\pm$1.31 
        & 71.74$\pm$0.28 
        & 74.92$\pm$0.88 
        & 57.11
        \\ 
        Herding               
        & 44.42$\pm$0.46 
        & 53.57$\pm$0.31 
        & 60.72$\pm$1.78 
        & 69.09$\pm$1.73 
        & 73.08$\pm$0.98 
        & 74.92$\pm$0.88 
        & 60.18
        \\ 
        Forgetting            
        & 26.39$\pm$0.17 
        & 40.78$\pm$2.02 
        & 49.95$\pm$2.31 
        & 65.71$\pm$1.12 
        & 73.67$\pm$1.12 
        & 74.92$\pm$0.88 
        & 51.30
        \\ 
        GraNd-score           
        & 36.33$\pm$2.66 
        & 46.21$\pm$1.48 
        & 55.51$\pm$0.76 
        & 64.59$\pm$2.40 
        & 70.14$\pm$1.36 
        & 74.92$\pm$0.88 
        & 54.56
        \\ 
        EL2N-score            
        & 21.64$\pm$2.03 
        & 23.78$\pm$1.66 
        & 35.71$\pm$1.17 
        & 56.32$\pm$0.86 
        & 69.66$\pm$0.43 
        & 74.92$\pm$0.88 
        & 41.42
        \\ 
        Optimization-based    
        & 33.42$\pm$1.60 
        & 45.37$\pm$2.81 
        & 54.06$\pm$1.74 
        & 65.19$\pm$1.27 
        & 70.06$\pm$0.83 
        & 74.92$\pm$0.88 
        & 54.42
        \\
        Self-sup.-selection   
        & 42.61$\pm$2.44 
        & 54.04$\pm$1.90 
        & 59.51$\pm$1.22 
        & 68.97$\pm$0.96 
        & 72.33$\pm$0.20 
        & 74.92$\pm$0.88 
        & 60.01
        \\ 
        Moderate-DS
        & 42.98$\pm$0.87
        & 55.80$\pm$0.95
        & 61.84$\pm$1.96
        & 70.05$\pm$1.29 
        & 73.67$\pm$0.30
        & 74.92$\pm$0.88
        & 60.87
        \\
        {\bf $\gm$ Matching}
        & {\bf 47.12$\pm$0.64}
        & {\bf 59.17$\pm$0.92}
        & {\bf 63.45$\pm$0.34}
        & {\bf 71.70$\pm$0.60}
        & {\bf 74.60$\pm$1.03}
        & 74.92$\pm$0.88
        & {\bf 63.21}
        \\
        \bottomrule
    \end{tabular}
}
    \caption{
    \footnotesize
    {\bf Image Corruption ( CIFAR 100 ):} Comparing (Test Accuracy) pruning methods when 20\% of the images are corrupted. ResNet-50 is used both as proxy and for downstream classification.}
    \label{tab:cifarC}
\end{table*}

\begin{table*}
%\footnotesize
\centering
\resizebox{\textwidth}{!}
{
    \begin{tabular}{lccccccc}
        \toprule
        \rowcolor{light-gray}
        \multicolumn{8}{c}{{\bf Tiny ImageNet}}
        \\
        \midrule
        {\textbf{Method / Ratio}} 
        & \textbf{20\%} 
        & \textbf{30\%} 
        & \textbf{40\%}
        & \textbf{60\%}
        & \textbf{80\%}
        & \textbf{100\%}
        & \textbf{Mean $\uparrow$}
        \\ 
        \midrule
        \multicolumn{8}{c}{{\bf No Corruption}}
        \\
        \midrule
        Random              
        & 24.02$\pm$0.41 
        & 29.79$\pm$0.27 
        & 34.41$\pm$0.46 
        & 40.96$\pm$0.47 
        & 45.74$\pm$0.61 
        & 49.36$\pm$0.25 
        & 34.98 
        \\ 
        Herding               
        & 24.09$\pm$0.45 
        & 29.39$\pm$0.53 
        & 34.13$\pm$0.37 
        & 40.86$\pm$0.61 
        & 45.45$\pm$0.33 
        & 49.36$\pm$0.25 
        & 34.78 
        \\ 
        Forgetting            
        & 22.37$\pm$0.71 
        & 28.67$\pm$0.54 
        & 33.64$\pm$0.32 
        & 41.14$\pm$0.43 
        & \textbf{46.77$\pm$0.31} 
        & 49.36$\pm$0.25 
        & 34.52
        \\ 
        GraNd-score           
        & 23.56$\pm$0.52 
        & 29.66$\pm$0.37 
        & 34.33$\pm$0.50 
        & 40.77$\pm$0.42 
        & 45.96$\pm$0.56 
        & 49.36$\pm$0.25 
        & 34.86 
        \\ 
        EL2N-score            
        & 19.74$\pm$0.26 
        & 26.58$\pm$0.40 
        & 31.93$\pm$0.28 
        & 39.12$\pm$0.46 
        & 45.32$\pm$0.27 
        & 49.36$\pm$0.25 
        & 32.54
        \\ 
        Optimization-based    
        & 13.88$\pm$2.17 
        & 23.75$\pm$1.62 
        & 29.77$\pm$0.94 
        & 37.05$\pm$2.81 
        & 43.76$\pm$1.50 
        & 49.36$\pm$0.25 
        & 29.64
        \\
        Self-sup.-selection   
        & 20.89$\pm$0.42 
        & 27.66$\pm$0.50 
        & 32.50$\pm$0.30 
        & 39.64$\pm$0.39 
        & 44.94$\pm$0.34 
        & 49.36$\pm$0.25 
        & 33.13
        \\ 
        Moderate-DS
        & 25.29$\pm$0.38
        & 30.57$\pm$0.20
        & 34.81$\pm$0.51
        & 41.45$\pm$0.44
        & 46.06$\pm$0.33 
        & 49.36$\pm$0.25 
        & 35.64
        \\ 
        {\bf $\gm$ Matching}  
        & {\bf 27.88$\pm$0.19}
        & {\bf 33.15$\pm$0.26}
        & {\bf 36.92$\pm$0.40}
        & {\bf 42.48$\pm$0.12}
        & 46.75$\pm$0.51
        & 49.36$\pm$0.25 
        & {\bf 37.44}
        \\
        \midrule
        \multicolumn{8}{c}{{\bf 5\% Feature Corruption}}
        \\
        \midrule
        Random              
        & 23.51$\pm$0.22 
        & 28.82$\pm$0.72 
        & 32.61$\pm$0.68 
        & 39.77$\pm$0.35 
        & 44.37$\pm$0.34 
        & 49.02$\pm$0.35 
        & 33.82
        \\ 
        Herding               
        & 23.09$\pm$0.53 
        & 28.67$\pm$0.37 
        & 33.09$\pm$0.32 
        & 39.71$\pm$0.31 
        & 45.04$\pm$0.15 
        & 49.02$\pm$0.35 
        & 33.92
        \\ 
        Forgetting            
        & 21.36$\pm$0.28 
        & 27.72$\pm$0.43 
        & 33.45$\pm$0.21 
        & 40.92$\pm$0.45
        & 45.99$\pm$0.51 
        & 49.02$\pm$0.35 
        & 33.89
        \\ 
        GraNd-score           
        & 22.47$\pm$0.23 
        & 28.85$\pm$0.83 
        & 33.81$\pm$0.24 
        & 40.40$\pm$0.15 
        & 44.86$\pm$0.49 
        & 49.02$\pm$0.35 
        & 34.08
        \\ 
        EL2N-score            
        & 18.98$\pm$0.72 
        & 25.96$\pm$0.28 
        & 31.07$\pm$0.63 
        & 38.65$\pm$0.36 
        & 44.21$\pm$0.68 
        & 49.02$\pm$0.35 
        & 31.77
        \\ 
        Optimization-based    
        & 13.65$\pm$1.26 
        & 24.02$\pm$1.35 
        & 29.65$\pm$1.86 
        & 36.55$\pm$1.84 
        & 43.64$\pm$0.71 
        & 49.02$\pm$0.35 
        & 29.50
        \\
        Self-sup.-selection   
        & 19.35$\pm$0.57 
        & 26.11$\pm$0.31 
        & 31.90$\pm$0.37 
        & 38.91$\pm$0.29 
        & 44.43$\pm$0.42 
        & 49.02$\pm$0.35 
        & 32.14
        \\ 
        Moderate-DS
        & 24.63$\pm$0.78
        & 30.27$\pm$0.16
        & 34.84$\pm$0.24
        & 40.86$\pm$0.42 
        & 45.60$\pm$0.31 
        & 49.02$\pm$0.35 
        & 35.24
        \\
        {\bf $\gm$ Matching}
        & {\bf 27.46$\pm$1.22}
        & {\bf 33.14$\pm$0.61}
        & {\bf 35.76$\pm$1.14}
        & {\bf 41.62$\pm$0.71}
        & {\bf 46.83$\pm$0.56}
        & 49.02$\pm$0.35 
        & {\bf 36.96}
        \\
        \midrule
        \multicolumn{8}{c}{{\bf 10\% Feature Corruption}}
        \\
        \midrule
        Random              
        & 22.67$\pm$0.27 
        & 28.67$\pm$0.52 
        & 31.88$\pm$0.30 
        & 38.63$\pm$0.36 
        & 43.46$\pm$0.20 
        & 48.40$\pm$0.32 
        & 33.06
        \\ 
        Herding               
        & 22.01$\pm$0.18 
        & 27.82$\pm$0.11 
        & 31.82$\pm$0.26 
        & 39.37$\pm$0.18 
        & 44.18$\pm$0.27 
        & 48.40$\pm$0.32 
        & 33.04
        \\ 
        Forgetting            
        & 20.06$\pm$0.48 
        & 27.17$\pm$0.36 
        & 32.31$\pm$0.22 
        & 40.19$\pm$0.29 
        & 45.51$\pm$0.48 
        & 48.40$\pm$0.32 
        & 33.05
        \\ 
        GraNd-score           
        & 21.52$\pm$0.48 
        & 26.98$\pm$0.43 
        & 32.70$\pm$0.19 
        & 40.03$\pm$0.26 
        & 44.87$\pm$0.35 
        & 48.40$\pm$0.32 
        & 33.22
        \\ 
        EL2N-score            
        & 18.59$\pm$0.13 
        & 25.23$\pm$0.18 
        & 30.37$\pm$0.22 
        & 38.44$\pm$0.32 
        & 44.32$\pm$1.07 
        & 48.40$\pm$0.32 
        & 31.39
        \\ 
        Optimization-based    
        & 14.05$\pm$1.74 
        & 29.18$\pm$1.77 
        & 29.12$\pm$0.61 
        & 36.28$\pm$1.88 
        & 43.52$\pm$0.31 
        & 48.40$\pm$0.32 
        & 29.03
        \\
        Self-sup.-selection   
        & 19.47$\pm$0.26 
        & 26.51$\pm$0.55 
        & 31.78$\pm$0.14 
        & 38.87$\pm$0.54 
        & 44.69$\pm$0.29 
        & 48.40$\pm$0.32 
        & 32.26
        \\ 
        Moderate-DS
        & 23.79$\pm$0.16 
        & 29.56$\pm$0.16 
        & 34.60$\pm$0.12 
        & 40.36$\pm$0.27 
        & 45.10$\pm$0.23 
        & 48.40$\pm$0.32 
        & 34.68
        \\
        {\bf $\gm$ Matching}
        & {\bf 27.41$\pm$0.23}
        & {\bf 32.84$\pm$0.98}
        & {\bf 36.27$\pm$0.68}
        & {\bf 41.85$\pm$0.29}
        & {\bf 46.35$\pm$0.44}
        & 48.40$\pm$0.32
        & {\bf 36.94}
        \\
        \midrule
        \multicolumn{8}{c}{{\bf 20\% Feature Corruption}}
        \\
        \midrule
        Random              
        & 19.99$\pm$0.42 
        & 25.93$\pm$0.53 
        & 30.83$\pm$0.44 
        & 37.98$\pm$0.31 
        & 42.96$\pm$0.62 
        & 46.68$\pm$0.43 
        & 31.54
        \\ 
        Herding               
        & 19.46$\pm$0.14 
        & 24.47$\pm$0.33 
        & 29.72$\pm$0.39 
        & 37.50$\pm$0.59 
        & 42.28$\pm$0.30 
        & 46.68$\pm$0.43 
        & 30.86
        \\ 
        Forgetting            
        & 18.47$\pm$0.46 
        & 25.53$\pm$0.23 
        & 31.17$\pm$0.24 
        & 39.35$\pm$0.44 
        & 44.55$\pm$0.67 
        & 46.68$\pm$0.43 
        & 31.81
        \\ 
        GraNd-score           
        & 20.07$\pm$0.49 
        & 26.68$\pm$0.40 
        & 31.25$\pm$0.40 
        & 38.21$\pm$0.49 
        & 42.84$\pm$0.72 
        & 46.68$\pm$0.43 
        & 30.53
        \\ 
        EL2N-score            
        & 18.57$\pm$0.30 
        & 24.42$\pm$0.44 
        & 30.04$\pm$0.15 
        & 37.62$\pm$0.44 
        & 42.43$\pm$0.61 
        & 46.68$\pm$0.43 
        & 30.53
        \\ 
        Optimization-based    
        & 13.71$\pm$0.26 
        & 23.33$\pm$1.84 
        & 29.15$\pm$2.84 
        & 36.12$\pm$1.86 
        & 42.94$\pm$0.52 
        & 46.88$\pm$0.43 
        & 29.06
        \\
        Self-sup.-selection   
        & 20.22$\pm$0.23 
        & 26.90$\pm$0.50 
        & 31.93$\pm$0.49 
        & 39.74$\pm$0.52 
        & 44.27$\pm$0.10 
        & 46.68$\pm$0.43 
        & 32.61
        \\ 
        Moderate-DS
        & 23.27$\pm$0.33 
        & 29.06$\pm$0.36 
        & 33.48$\pm$0.11 
        & 40.07$\pm$0.36 
        & 44.73$\pm$0.39 
        & 46.68$\pm$0.43 
        & 34.12
        \\
        {\bf $\gm$ Matching}
        & {\bf 27.19$\pm$0.92}
        & {\bf 31.70$\pm$0.78}
        & {\bf 35.14$\pm$0.19}
        & {\bf 42.04$\pm$0.31}
        & {\bf 45.12$\pm$0.28} 
        & 46.68$\pm$0.43 
        & {\bf 36.24}
        \\
        \bottomrule
    \end{tabular}
}
    \caption
    {
    \footnotesize
    {\bf Image Corruption ( Tiny ImageNet ):} Comparing (Test Accuracy) pruning methods under feature (image) corruption. ResNet-50 is used both as proxy and for downstream classification.
    }
    \label{tab:tinyC}
\end{table*}

\subsection{Corruption Scenarios}
To understand the performance of data pruning strategies in presence of corruption, we experiment with three different sources of corruption -- image corruption, label noise and adversarial attacks. 

\subsubsection{Robustness to Image Corruption} 
In this set of experiments, we investigate the robustness of data pruning strategies when the input images are corrupted -- a popular robustness setting, often encountered when training models on real-world data~\citep{hendrycks2019benchmarking, szegedy2013intriguing}.
To corrupt images, we apply five types of realistic noise: Gaussian noise, random occlusion, resolution reduction, fog, and motion blur to parts of the corrupt samples i.e. to say if $m$ samples are corrupted, each type of noise is added to one a random $m/5$ of them, while the other partitions are corrupted with a different noise. The results are presented in~\cref{tab:cifarC},~\ref{tab:tinyC}. We observe that $\gm$ Matching outperforms all the baselines across all pruning rates improving $\approx $3\% across both datasets on an average. We note that, the gains are more consistent and profound in this setting over the clean setting. Additionally, similar to our prior observations in the clean setting, the gains of $\gm$ Matching are more significant at high pruning rates (i.e. low selection ratio).

\begin{table*}
    \footnotesize
    \centering
    \begin{tabular}{lccccc}
        \toprule
        
        & \multicolumn{2}{c}{\bf CIFAR-100 (Label noise)} 
        & \multicolumn{2}{c}{\bf Tiny ImageNet (Label noise)} 
        & 
        \\
        \cmidrule(r){2-3} 
        \cmidrule(r){4-5} 
        {\bf Method / Ratio }
        & {\bf 20\%} 
        & {\bf 30\%} 
        & {\bf 20\%} 
        & {\bf 30\%} 
        & {\bf Mean $\uparrow$}
        \\
        \midrule
        \multicolumn{6}{c}{\bf 20\% Label Noise}
        \\
        \midrule
        Random 
        & 34.47$\pm$0.64 
        & 43.26$\pm$1.21 
        & 17.78$\pm$0.44 
        & 23.88$\pm$0.42 
        & 29.85
        \\
        Herding 
        & 42.29$\pm$1.75 
        & 50.52$\pm$3.38
        & 18.98$\pm$0.44 
        & 24.23$\pm$0.29 
        &  34.01
        \\
        Forgetting 
        & 36.53$\pm$1.11 
        & 45.78$\pm$1.04 
        & 13.20$\pm$0.38 
        & 21.79$\pm$0.43 
        & 29.33
        \\
        GraNd-score 
        & 31.72$\pm$0.67 
        & 42.80$\pm$0.30 
        & 18.28$\pm$0.32 
        & 23.72$\pm$0.18 
        & 28.05
        \\
        EL2N-score 
        & 29.82$\pm$1.19 
        & 33.62$\pm$2.35 
        & 13.93$\pm$0.69 
        & 18.57$\pm$0.31 
        & 23.99
        \\
        Optimization-based 
        & 32.79$\pm$0.62 
        & 41.80$\pm$1.14
        & 14.77$\pm$0.95 
        & 22.52$\pm$0.77 
        & 27.57
        \\
        Self-sup.-selection 
        & 31.08$\pm$0.78 
        & 41.87$\pm$0.63 
        & 15.10$\pm$0.73 
        & 21.01$\pm$0.36 
        & 27.27
        \\
        Moderate-DS 
        & 40.25$\pm$0.12 
        & 48.53$\pm$1.60 
        & 19.64$\pm$0.40 
        & 24.96$\pm$0.30
        & 31.33
        \\
        {\bf $\gm$ Matching}
        & {\bf 52.64$\pm$0.72}
        & {\bf 61.01$\pm$0.47}
        & {\bf 25.80$\pm$0.37} 
        & {\bf 31.71$\pm$0.24}
        & {\bf 42.79}
        \\
        \midrule
        \multicolumn{6}{c}{\bf 35\% Label Noise}
        \\
        \midrule
        Random 
        & 24.51$\pm$1.34 
        & 32.26$\pm$0.81 
        & 14.64$\pm$0.29 
        & 19.41$\pm$0.45 
        & 22.71
        \\
        Herding 
        & 29.42$\pm$1.54 
        & 37.50$\pm$2.12 
        & 15.14$\pm$0.45 
        & 20.19$\pm$0.45 
        & 25.56
        \\
        Forgetting 
        & 29.48$\pm$1.98 
        & 38.01$\pm$2.21 
        & 11.25$\pm$0.90 
        & 17.07$\pm$0.66 
        & 23.14
        \\
        GraNd-score 
        & 23.03$\pm$1.05 
        & 34.83$\pm$2.01 
        & 13.68$\pm$0.46 
        & 19.51$\pm$0.45 
        & 22.76
        \\
        EL2N-score 
        & 21.95$\pm$1.08 
        & 31.63$\pm$2.84 
        & 10.11$\pm$0.25 
        & 13.69$\pm$0.32 
        & 19.39
        \\
        Optimization-based 
        & 26.77$\pm$0.15 
        & 35.63$\pm$0.92 
        & 12.37$\pm$0.68 
        & 18.52$\pm$0.90 
        & 23.32
        \\
        Self-sup.-selection 
        & 23.12$\pm$1.47 
        & 34.85$\pm$0.68 
        & 11.23$\pm$0.32 
        & 17.76$\pm$0.69 
        & 22.64
        \\
        Moderate-DS 
        & 28.45$\pm$0.53 
        & 36.55$\pm$1.26 
        & 15.27$\pm$0.31 
        & 20.33$\pm$0.28 
        & 25.15
        \\
        {\bf $\gm$ Matching}
        & {\bf 43.33$\pm$ 1.02}
        & {\bf 58.41$\pm$ 0.68}
        & {\bf 23.14$\pm$ 0.92} 
        & {\bf 27.76$\pm$ 0.40}
        & {\bf 38.16}
        \\
        \bottomrule
    \end{tabular}
    \caption{
    \footnotesize
    {\bf Robustness to Label Noise:} Comparing (Test Accuracy) pruning methods on CIFAR-100 and TinyImageNet datasets, under 20\% and 35\% Symmetric Label Corruption, at 20\% and 30\% selection ratio. ResNet-50 is used both as proxy and for downstream classification.}
    \label{tab:label-noise}
\end{table*}
\begin{table*}
\footnotesize
\centering
\resizebox{\textwidth}{!}
{
    \begin{tabular}{lccccccc}
        \midrule
        \rowcolor{light-gray}
        \multicolumn{8}{c}{{\bf Tiny ImageNet (Label Noise)}}
        \\
        \midrule
        {\textbf{Method / Ratio}} 
        & \textbf{20\%} 
        & \textbf{30\%} 
        & \textbf{40\%}
        & \textbf{60\%}
        & \textbf{80\%}
        & \textbf{100\%}
        & \textbf{Mean $\uparrow$}
        \\ 
        \midrule
        Random 
        & 17.78$\pm$0.44 
        & 23.88$\pm$0.42 
        & 27.97$\pm$0.39 
        & 34.88$\pm$0.51 
        & 38.47$\pm$0.40 
        & 44.42$\pm$0.47 
        & 28.60
        \\
        Herding 
        & 18.98$\pm$0.44 
        & 24.23$\pm$0.29 
        & 27.28$\pm$0.31 
        & 34.36$\pm$0.29 
        & 39.00$\pm$0.49 
        & 44.42$\pm$0.47 
        & 28.87
        \\
        Forgetting 
        & 13.20$\pm$0.38 
        & 21.79$\pm$0.43 
        & 27.89$\pm$0.22 
        & \textbf{36.03$\pm$0.24} 
        & \textbf{40.60$\pm$0.31} 
        & 44.42$\pm$0.47 
        & 27.50
        \\
        GraNd-score 
        & 18.28$\pm$0.32 
        & 23.72$\pm$0.18 
        & 27.34$\pm$0.33 
        & 34.91$\pm$0.19 
        & 39.45$\pm$0.45 
        & 44.42$\pm$0.47 
        & 28.34
        \\
        EL2N-score 
        & 13.93$\pm$0.69 
        & 18.57$\pm$0.31 
        & 24.56$\pm$0.34 
        & 32.14$\pm$0.49 
        & 37.64$\pm$0.41 
        & 44.42$\pm$0.47 
        & 25.37
        \\
        Optimization-based 
        & 14.77$\pm$0.95 
        & 22.52$\pm$0.77 
        & 25.62$\pm$0.90 
        & 34.18$\pm$0.79 
        & 38.49$\pm$0.69 
        & 44.42$\pm$0.47 
        & 27.12
        \\
        Self-sup.-selection 
        & 15.10$\pm$0.73 
        & 21.01$\pm$0.36 
        & 26.62$\pm$0.22 
        & 33.93$\pm$0.36 
        & 39.22$\pm$0.12 
        & 44.42$\pm$0.47 
        & 27.18
        \\
        Moderate-DS 
        & \textbf{19.64$\pm$0.40} 
        & \textbf{24.96$\pm$0.30} 
        & \textbf{29.56$\pm$0.21} 
        & 35.79$\pm$0.36 
        & 39.93$\pm$0.23 
        & 44.42$\pm$0.47 
        & 30.18
        \\
        {\bf $\gm$ Matching}
        & {\bf 25.80$\pm$0.37} 
        & {\bf 31.71$\pm$0.24}
        & {\bf 34.87$\pm$0.21}
        & {\bf 39.76$\pm$0.71}
        & {\bf 41.94$\pm$0.23}
        & 44.42$\pm$0.47
        & {\bf 34.82}
        \\
        \bottomrule
    \end{tabular} 
}
    \caption
    {
    \footnotesize
    {\bf Pruning with Label Noise (TinyImageNet):} Comparing (Test Accuracy) pruning methods under 20\% Symmetric Label Corruption across wide array of selection ratio. ResNet-50 is used both as proxy and for downstream classification.
    } 
    \label{tab:tinyN} 
\end{table*}

\subsubsection{Robustness to Label Corruption}
Next, we consider another important corruption scenario where a fraction of the training examples are mislabeled. Since acquiring manually labeled data is often impractical and data in the wild always contains noisy annotations -- it is important to ensure the pruning method is robust to label noise. We conduct experiments with synthetically injected symmetric label noise~\citep{li2022selective, patrini2017making, xia2020robust}. The results are summarized in~\cref{tab:label-noise},\ref{tab:tinyN}. Encouragingly, $\gm$ Matching {\bf outperforms the baselines by $\approx$ 12\%}. Since, mislabeled samples come from different class - they are spatially quite dissimilar than the correctly labeled ones thus, less likely to be picked by $\gm$ matching, explaining the superior performance.

\begin{table*}
    \footnotesize
    \centering
    \begin{tabular}{lccccc}
        \toprule
        & \multicolumn{2}{c}{\bf CIFAR-100 (PGD Attack)} 
        & \multicolumn{2}{c}{\bf CIFAR-100 (GS Attack)} 
        & \\
        \cmidrule(r){2-3} 
        \cmidrule(r){4-5} 
        {\bf Method / Ratio } 
        & 20\% 
        & 30\% 
        & 20\% 
        & 30\% 
        & {\bf Mean $\uparrow$} 
        \\
        \midrule
        Random 
        & 43.23$\pm$0.31 
        & 52.86$\pm$0.34 
        & 44.23$\pm$0.41 
        & 53.44$\pm$0.44 
        & 48.44
        \\
        Herding 
        & 40.21$\pm$0.72 
        & 49.62$\pm$0.65 
        & 39.92$\pm$1.03 
        & 50.14$\pm$0.15 
        & 44.97
        \\
        Forgetting 
        & 35.90$\pm$1.30 
        & 47.37$\pm$0.99 
        & 37.55$\pm$0.53 
        & 46.88$\pm$1.91 
        & 41.93
        \\
        GraNd-score 
        & 40.87$\pm$0.84 
        & 50.13$\pm$0.30 
        & 40.77$\pm$1.11 
        & 49.88$\pm$0.83 
        & 45.41
        \\
        EL2N-score 
        & 26.61$\pm$0.58 
        & 34.50$\pm$1.02 
        & 26.72$\pm$0.66 
        & 35.55$\pm$1.30 
        & 30.85
        \\
        Optimization-based 
        & 38.29$\pm$1.77 
        & 46.25$\pm$1.82 
        & 41.36$\pm$0.92 
        & 49.10$\pm$0.81 
        & 43.75
        \\
        Self-sup.-selection 
        & 40.53$\pm$1.15 
        & 49.95$\pm$0.50 
        & 40.74$\pm$1.66 
        & 51.23$\pm$0.25 
        & 45.61
        \\
        Moderate-DS 
        & 43.60$\pm$0.97 
        & 51.66$\pm$0.39 
        & 44.69$\pm$0.68 
        & 53.71$\pm$0.37 
        & 48.42
        \\
        {\bf $\gm$ Matching}
        & {\bf 45.41 $\pm$0.86}
        & {\bf 51.80 $\pm$1.01}
        & {\bf 49.78 $\pm$0.27}
        & {\bf 55.50 $\pm$0.31}
        & {\bf 50.62}
        \\
        \toprule
        & \multicolumn{2}{c}{\bf Tiny ImageNet (PGD Attack)} 
        & \multicolumn{2}{c}{\bf Tiny ImageNet (GS Attack)} 
        & \\
        \cmidrule(r){2-3} 
        \cmidrule(r){4-5} 
        {\bf Method / Ratio } 
        & 20\% 
        & 30\% 
        & 20\% 
        & 30\% 
        & {\bf Mean $\uparrow$} 
        \\
        \midrule
        Random 
        & 20.93$\pm$0.30 
        & 26.60$\pm$0.98 
        & 22.43$\pm$0.31 
        & 26.89$\pm$0.31 
        & 24.21
        \\
        Herding 
        & 21.61$\pm$0.36 
        & 25.95$\pm$0.19 
        & 23.04$\pm$0.28 
        & 27.39$\pm$0.14 
        & 24.50
        \\
        Forgetting 
        & 20.38$\pm$0.47 
        & 26.12$\pm$0.19 
        & 22.06$\pm$0.31 
        & 27.21$\pm$0.21 
        & 23.94
        \\
        GraNd-score 
        & 20.76$\pm$0.21 
        & 26.34$\pm$0.32 
        & 22.56$\pm$0.30 
        & 27.52$\pm$0.40 
        & 24.30
        \\
        EL2N-score 
        & 16.67$\pm$0.62 
        & 22.36$\pm$0.42 
        & 19.93$\pm$0.57 
        & 24.65$\pm$0.32 
        & 20.93
        \\
        Optimization-based 
        & 19.26$\pm$0.77 
        & 24.55$\pm$0.92 
        & 21.26$\pm$0.24 
        & 25.88$\pm$0.37 
        & 22.74
        \\
        Self-sup.-selection 
        & 19.23$\pm$0.46 
        & 23.92$\pm$0.51 
        & 19.70$\pm$0.20 
        & 24.73$\pm$0.39 
        & 21.90
        \\
        Moderate-DS 
        & 21.81$\pm$0.37 
        & 27.11$\pm$0.20 
        & 23.20$\pm$0.13 
        & 28.89$\pm$0.27 
        & 25.25
        \\
        {\bf $\gm$ Matching}
        & {\bf 25.98 $\pm$1.12}
        & {\bf 30.77 $\pm$0.25}
        & {\bf 29.71 $\pm$0.45}
        & {\bf 32.88 $\pm$0.73}
        & {\bf 29.84}
        \\
        \bottomrule
    \end{tabular}
    \caption{
    \footnotesize
    {\bf Robustness to Adversarial Attacks}. Comparing (Test Accuracy) pruning methods under PGD and GS attacks. ResNet-50 is used both as proxy and for downstream classification.}
    \label{tab:Adv-A}
\end{table*}

\subsubsection{Robustness to Adversarial Attacks}

Finally, we investigate the robustness properties of data pruning strategies in presence of adversarial attacks that add imperceptible but adversarial noise on natural examples~\citep{szegedy2013intriguing, ma2018characterizing, huang2010active}. Specifically, we experiment with two popular adversarial attack algorithms -- PGD attack~\citep{madry2017towards} and GS Attacks~\citep{goodfellow2014explaining}. We employ adversarial attacks on models trained with CIFAR-100 and Tiny-ImageNet to generate adversarial examples. Following this, various methods are applied to these adversarial examples, and the models are retrained on the curated subset of data. The results are summarized in~\cref{tab:Adv-A}. Similar to other corruption scenarios, even in this setting, $\gm$ Matching outperforms the baselines $\approx$  and remains effective across different pruning rates yielding $\approx$ 3\% average gain over the best performing baseline.

\begin{table*}
    \footnotesize
    \centering
    \begin{tabular}{lccccc}
        \toprule
        
        & \multicolumn{2}{c}{\bf ResNet-50→SENet}
        & \multicolumn{2}{c}{\bf ResNet-50→EfficientNet-B0} 
        & \\
        \cmidrule(r){2-3} 
        \cmidrule(r){4-5} 
        {\bf Method / Ratio } 
        & {\bf 20\%} 
        & {\bf 30\%} 
        & {\bf 20\%} 
        & {\bf 30\%} 
        & {\bf Mean $\uparrow$} 
        \\
        \midrule
        Random 
        & 34.13$\pm$0.71 
        & 39.57$\pm$0.53 
        & 32.88$\pm$1.52 
        & 39.11$\pm$0.94 
        & 36.42 
        \\ 
        Herding 
        & 34.86$\pm$0.55 
        & 38.60$\pm$0.68 
        & 32.21$\pm$1.54 
        & 37.53$\pm$0.22 
        & 35.80
        \\
        Forgetting 
        & 33.40$\pm$0.64 
        & 39.79$\pm$0.78 
        & 31.12$\pm$0.21 
        & 38.38$\pm$0.65 
        &  35.67
        \\
        GraNd-score 
        & 35.12$\pm$0.54 
        & 41.14$\pm$0.42 
        & 33.20$\pm$0.67 
        & 40.02$\pm$0.35
        & 37.37
        \\
        EL2N-score 
        & 31.08$\pm$1.11 
        & 38.26$\pm$0.45 
        & 31.34$\pm$0.49 
        & 36.88$\pm$0.32 
        & 34.39
        \\
        Optimization-based 
        & 33.18$\pm$0.52 
        & 39.42$\pm$0.77 
        & 32.16$\pm$0.90 
        & 38.52$\pm$0.50 
        & 35.82
        \\
        Self-sup.-selection 
        & 31.74$\pm$0.71 
        & 38.45$\pm$0.39 
        & 30.99$\pm$1.03 
        & 37.96$\pm$0.77 
        & 34.79
        \\
        Moderate-DS 
        & 36.04$\pm$0.15
        & 41.40$\pm$0.20 
        & 34.26$\pm$0.48 
        & 39.57$\pm$0.29 
        & 37.82
        \\
        {\bf $\gm$ Matching} 
        & {\bf 37.93$\pm$0.23} 
        & {\bf 42.59$\pm$0.29}
        & {\bf 36.31$\pm$0.67}
        & {\bf 41.03$\pm$0.41} 
        & {\bf 39.47}
        \\
        \bottomrule
    \end{tabular}
    \caption
    {
    \footnotesize
    {\bf Network Transfer (Clean)} : Tiny-ImageNet Model Transfer Results. A ResNet-50 proxy is used to find important samples which are then used to train SENet and EfficientNet.
    }
    \label{tab:clean-nw-transfer}
\end{table*}

\begin{table*}
    \footnotesize
    \centering
    \begin{tabular}{lccccc}
        \toprule
        & \multicolumn{2}{c}{\bf ResNet-50→ VGG-16} 
        & \multicolumn{2}{c}{\bf ResNet-50→ ShuffleNet} 
        & \\
        \cmidrule(r){2-3} 
        \cmidrule(r){4-5} 
        {\bf Method / Ratio } 
        & 20\% 
        & 30\% 
        & 20\% 
        & 30\% 
        & {\bf Mean $\uparrow$} 
        \\
        \midrule
        \multicolumn{6}{c}{\bf No Corruption} 
        \\
        \midrule 
        Random 
        & 29.63$\pm$0.43 
        & 35.38$\pm$0.83 
        & 32.40$\pm$1.06 
        & 39.13$\pm$0.81 
        & 34.96
        \\
        Herding 
        & 31.05$\pm$0.22 
        & 36.27$\pm$0.57 
        & 33.10$\pm$0.39 
        & 38.65$\pm$0.22 
        & 35.06
        \\
        Forgetting 
        & 27.53$\pm$0.36 
        & 35.61$\pm$0.39 
        & 27.82$\pm$0.56 
        & 36.26$\pm$0.51 
        & 32.35
        \\
        GraNd-score 
        & 29.93$\pm$0.95 
        & 35.61$\pm$0.39 
        & 29.56$\pm$0.46 
        & 37.40$\pm$0.38 
        & 33.34
        \\
        EL2N-score 
        & 26.47$\pm$0.31 
        & 33.19$\pm$0.51 
        & 28.18$\pm$0.27 
        & 35.81$\pm$0.29 
        & 31.13
        \\
        Optimization-based 
        & 25.92$\pm$0.64 
        & 34.82$\pm$1.29 
        & 31.37$\pm$1.14 
        & 38.22$\pm$0.78 
        & 32.55
        \\
        Self-sup.-selection 
        & 25.16$\pm$1.10 
        & 33.30$\pm$0.94 
        & 29.47$\pm$0.56 
        & 36.68$\pm$0.36 
        & 31.45
        \\
        Moderate-DS 
        & 31.45$\pm$0.32 
        & 37.89$\pm$0.36 
        & 33.32$\pm$0.41 
        & 39.68$\pm$0.34
        & 35.62
        \\
        {\bf $\gm$ Matching}
        & {\bf 35.86$\pm$0.41}
        & {\bf 40.56$\pm$0.22}
        & {\bf 35.51$\pm$0.32}
        & {\bf 40.30$\pm$0.58}
        & {\bf 38.47}
        \\
        \midrule
        \multicolumn{6}{c}{\bf 20\% Label Corruption} 
        \\
        \midrule 
        Random 
        & 23.29$\pm$1.12 
        & 28.18$\pm$1.84 
        & 25.08$\pm$1.32 
        & 31.44$\pm$1.21 
        & 27.00
        \\
        Herding 
        & 23.99$\pm$0.36
        & 28.57$\pm$0.40
        & 26.25$\pm$0.47 
        & 30.73$\pm$0.28
        & 27.39
        \\
        Forgetting 
        & 14.52$\pm$0.66 
        & 21.75$\pm$0.23 
        & 15.70$\pm$0.29 
        & 22.31$\pm$0.35 
        & 18.57
        \\
        GraNd-score 
        & 22.44$\pm$0.46 
        & 27.95$\pm$0.29 
        & 23.64$\pm$0.10 
        & 30.85$\pm$0.21 
        & 26.22
        \\
        EL2N-score 
        & 15.15$\pm$1.25 
        & 23.36$\pm$0.30 
        & 18.01$\pm$0.44 
        & 24.68$\pm$0.34 
        & 20.30
        \\
        Optimization-based 
        & 22.93$\pm$0.58 
        & 24.92$\pm$2.50 
        & 25.82$\pm$1.70 
        & 30.19$\pm$0.48 
        & 25.97
        \\
        Self-sup.-selection 
        & 18.39$\pm$1.30 
        & 25.77$\pm$0.87 
        & 22.87$\pm$0.54 
        & 29.80$\pm$0.36 
        & 24.21
        \\
        Moderate-DS 
        & 23.68$\pm$0.19 
        & 28.93$\pm$0.19 
        & 28.82$\pm$0.33 
        & 32.39$\pm$0.21 
        & 28.46
        \\
        {\bf $\gm$ Matching}
        & {\bf 28.77$\pm$0.77}
        & {\bf 34.87$\pm$0.23}
        & {\bf 32.05$\pm$0.93}
        & {\bf 37.43$\pm$0.25}
        & {\bf 33.28}
        \\
        \midrule
        \multicolumn{6}{c}{\bf 20\% Feature Corruption} 
        \\
        \midrule
        Random 
        & 26.33$\pm$0.88 
        & 31.57$\pm$1.31 
        & 29.15$\pm$0.83 
        & 34.72$\pm$1.00 
        & 30.44
        \\
        Herding 
        & 18.03$\pm$0.33 
        & 25.77$\pm$0.34 
        & 23.33$\pm$0.43 
        & 31.73$\pm$0.38 
        & 24.72
        \\
        Forgetting 
        & 19.41$\pm$0.57 
        & 28.35$\pm$0.16 
        & 18.44$\pm$0.57 
        & 31.09$\pm$0.61 
        & 24.32
        \\
        GraNd-score 
        & 23.59$\pm$0.19 
        & 30.69$\pm$0.13 
        & 23.15$\pm$0.56 
        & 31.58$\pm$0.95 
        & 27.25
        \\
        EL2N-score 
        & 24.60$\pm$0.81 
        & 31.49$\pm$0.33 
        & 26.62$\pm$0.34 
        & 33.91$\pm$0.56 
        & 29.16
        \\
        Optimization-based 
        & 25.12$\pm$0.34 
        & 30.52$\pm$0.89 
        & 28.87$\pm$1.25 
        & 34.08$\pm$1.92 
        & 29.65
        \\
        Self-sup.-selection 
        & 26.33$\pm$0.21 
        & 33.23$\pm$0.26 
        & 26.48$\pm$0.37 
        & 33.54$\pm$0.46 
        & 29.90
        \\
        Moderate-DS 
        & 29.65$\pm$0.68 
        & 35.89$\pm$0.53 
        & 32.30$\pm$0.38
        & 38.66$\pm$0.29 
        & 34.13
        \\
        $\gm$ Matching
        & {\bf 33.45$\pm$1.02}
        & {\bf 39.46$\pm$0.44}
        & {\bf 35.14$\pm$0.21}
        & {\bf 39.89$\pm$0.98}
        & {\bf 36.99}
        \\
        \midrule
        \multicolumn{6}{c}{\bf PGD Attack} 
        \\
        \midrule
        Random 
        & 26.12$\pm$1.09 
        & 31.98$\pm$0.78 
        & 28.28$\pm$0.90 
        & 34.59$\pm$1.18 
        & 30.24
        \\
        Herding 
        & 26.76$\pm$0.59 
        & 32.56$\pm$0.35 
        & 28.87$\pm$0.48 
        & 35.43$\pm$0.22 
        & 30.91
        \\
        Forgetting 
        & 24.55$\pm$0.57 
        & 31.83$\pm$0.36 
        & 23.32$\pm$0.37 
        & 31.82$\pm$0.15 
        & 27.88
        \\
        GraNd-score 
        & 25.19$\pm$0.33 
        & 31.46$\pm$0.54 
        & 26.03$\pm$0.66 
        & 33.22$\pm$0.24 
        & 28.98
        \\
        EL2N-score 
        & 21.73$\pm$0.47 
        & 27.66$\pm$0.32 
        & 22.66$\pm$0.35 
        & 29.89$\pm$0.64 
        & 25.49
        \\
        Optimization-based 
        & 26.02$\pm$0.36 
        & 31.64$\pm$1.75 
        & 27.93$\pm$0.47 
        & 34.82$\pm$0.96 
        & 30.10
        \\
        Self-sup.-selection 
        & 22.36$\pm$0.30 
        & 28.56$\pm$0.50 
        & 25.35$\pm$0.27 
        & 32.57$\pm$0.13 
        & 27.21
        \\
        Moderate-DS 
        & 27.24$\pm$0.36 
        & 32.90$\pm$0.31 
        & 29.06$\pm$0.28 
        & 35.89$\pm$0.53 
        & 31.27
        \\
        {\bf $\gm$ Matching}
        & {\bf 27.96$\pm$1.60}
        & {\bf 35.76$\pm$0.82}
        & {\bf 34.11$\pm$0.65}
        & {\bf 40.91$\pm$0.84}
        & {\bf 34.69}
        \\
        \bottomrule
    \end{tabular}
    \caption{
    \footnotesize
    {\bf Network Transfer} : A ResNet-50 proxy (pretrained on TinyImageNet) is used to find important samples from Tiny-ImageNet; which is then used to train a VGGNet-16 and ShuffleNet. We repeat the experiment across multiple corruption settings - clean; 20\% Feature / Label Corruption and PGD attack when 20\% and 30\% samples are selected.}
    \label{tab:VGG-Shuffle}
\end{table*}

\subsection{Generalization to Unseen Network / Domain}
\label{sec:proxy_ablations}
One crucial component of data pruning is the proxy network. Since, the input features (e.g. images) often reside on a non-separable manifold, data pruning strategies rely on a proxy model to map the samples into a separable manifold (embedding space), wherein the data pruning strategies can now assign importance scores. However, it is important for the data pruning strategies to be robust to architecture changes i.e. to say that samples selected via a proxy network should generalize well when trained on unseen (during sample selection) networks / domains. We perform experiments on two such scenarios: 

{\bf A. Network Transfer:} In this setting, the proxy model is trained on the target dataset (no distribution shift). However, the proxy architecture is different than the downstream network. In~\cref{tab:clean-nw-transfer}, we use a ResNet-50 proxy trained on MiniImageNet to sample the data. However, then we train a downstram SENet and EfficientNet-B0 on the sampled data. In~\cref{tab:VGG-Shuffle}, we use a ResNet-50 pretrained on Mini ImageNet as proxy, whereas we train a VGG-16 and ShuffleNet over the selected samples.

{\bf B. Domain Transfer:} Next, we consider the setting where the proxy shares the same architecture with the downstream model. However, the proxy used to select the samples is pretrained on a different dataset (distribution shift) than target dataset. In~\cref{fig:main-cifar-vanilla} we use a proxy ResNet-18 pretrained on ImageNet to select samples from CIFAT-10. The selected samples are used to train a subsequent ResNet-18. In~\cref{fig:main-cifar-lp}, we additionally freeze the pretrained encoder i.e. we use ResNet-18 encoder pretrained on ImageNet as proxy. Further, we freeze the encoder and train a downstream linear classifier on top over CIFAR-10 (linear probe). 

\section{Complete Proofs}
In this section we will describe the proof techniques involved in establishing the theoretical claims in the paper. 

\subsubsection{Intermediate Lemmas}
% We restate~\cref{lemma:gm_mean_bound}for convenience:

In order to prove~\cref{th:gm_match_cvg}, we will first establish the following result which follows from the definition of $\gm$; see also~\citep{lopuhaa1991breakdown,minsker2015geometric,cohen2016geometric,chen2017distributed,li2019rsa,wu2020federated, pmlr-v151-acharya22a} for similar adaptations.

\begin{lemma}
    Given a set of $\alpha$-corrupted samples $\gD = \gD_\gG \cup \gD_\gB$ (~\cref{def:corruption_model}), and an $\epsilon$-approx. $\gm(\cdot)$ oracle~\eqref{eq:gm}, then we have: 
    \begin{equation}
        \E\bigg\| \bm{\mu}^\gm - \bm{\mu}^\gG\bigg \|^2\leq 
        \frac{8|\gD_\gG|}{(|\gD_\gG|-|\gD_\gB|)^2}\sum_{\vx \in \gD_\gG}\E\bigg\|\vx -\bm{\mu}^\gG \bigg\|^2 +\frac{2\epsilon^2}{(|\gD_\gG|-|\gD_\gB|)^2}
        \label{eq:gm_mean_bound}
     \end{equation}
     where, $\bm{\mu}^\gm = \gm(\{\vx_i \in \gD\})$ is the $\epsilon$-approximate $\gm$ over the entire ($\alpha$-corrupted) dataset; and $\bm{\mu}^\gG = \frac{1}{|\gD_\gG|}\sum_{\vx_i \in \gD_\gG}\vx_i$ denotes the mean of the (underlying) uncorrupted set. 
\label{lemma:gm_mean_bound}
\end{lemma} 

\begin{proof}
Note that, by using triangle inequality, we can write:
\begin{flalign}
    \sum_{\vx_i \in \gD} \bigg\|\bm{\mu}^\gm - \vx_i\bigg\| 
    \nonumber
    & \geq \sum_{\vx_i \in \gD_\gB} \bigg( \bigg\|\vx_i\bigg\| - \bigg\|\bm{\mu}^\gm\bigg\|\bigg) +
    \sum_{\vx_i \in \gD_\gG} \bigg( \bigg\|\bm{\mu}^\gm\bigg\| - \bigg\|\vx_i\bigg\|\bigg)\\
    \nonumber
    & = \bigg( \sum_{\vx_i \in \gD_\gG} - \sum_{\vx_i \in \gD_\gB} \bigg) \bigg\|\bm{\mu}^\gm\bigg\| +
    \sum_{\vx_i \in \gD_\gB} \bigg\|\vx_i\bigg\| - \sum_{\vx_i \in \gD_\gG} \bigg\|\vx_i\bigg\|\\
    & = \bigg( |\gD_\gG| - |\gD_\gB|\bigg) \bigg\|\bm{\mu}^\gm\bigg\| +
    \sum_{\vx_i \in \gD} \bigg\|\vx_i\bigg\| - 2\sum_{\vx_i \in \gD_\gG} \bigg\|\vx_i\bigg\|.
\end{flalign}
    
Now, by definition~\eqref{eq:gm_epsilon}; we have that:
\begin{equation}
    \sum_{\vx_i \in \gD} \bigg\|\bm{\mu}^\gm - \vx_i\bigg\| 
    \leq \inf_{\vz \in \gH} \sum_{\vx_i \in \gD} \bigg\| \vz - \vx_i\bigg\| + \epsilon
    \leq \sum_{\vx_i \in \gD} \bigg\|\vx_i\bigg\| + \epsilon
\end{equation}

Combining these two inequalities, we get:
\begin{flalign}
    \bigg( |\gD_\gG| - |\gD_\gB|\bigg) \bigg\|\bm{\mu}^\gm\bigg\|
    \leq \sum_{\vx_i \in \gD} \bigg\|\vx_i\bigg\| - \sum_{\vx_i \in \gD} \bigg\|\vx_i\bigg\| + 2 \sum_{\vx_i \in \gD_\gG} \bigg\|\vx_i\bigg\| + \epsilon
\end{flalign}
This implies: 
\begin{flalign}
    \bigg\|\bm{\mu}^\gm\bigg\| \leq \frac{2 }{\bigg(|\gD_\gG| - |\gD_\gB|\bigg)}\sum_{\vx_i \in \gD_\gG} \bigg\|\vx_i\bigg\| + \frac{\epsilon}{\bigg( |\gD_\gG| - |\gD_\gB|\bigg)}
\end{flalign}
Squaring both sides, 
\begin{flalign}
    \bigg\|\bm{\mu}^\gm\bigg\|^2 
    &\leq \Bigg[\frac{2}{\bigg(|\gD_\gG| - |\gD_\gB|\bigg)}\sum_{\vx_i \in \gD_\gG} \bigg\|\vx_i\bigg\| + \frac{\epsilon}{\bigg( |\gD_\gG| - |\gD_\gB|\bigg)}\Bigg]^2\\
    &\leq 2\Bigg[\frac{2}{\bigg(|\gD_\gG| - |\gD_\gB|\bigg)}\sum_{\vx_i \in \gD_\gG} \bigg\|\vx_i\bigg\|\Bigg]^2 + 2 \Bigg[ \frac{\epsilon}{\bigg( |\gD_\gG| - |\gD_\gB|\bigg)}\Bigg]^2
\end{flalign}
Where, the last step is a well-known consequence of triangle inequality and AM-GM inequality. 
Taking expectation on  both sides, we have:
\begin{flalign}
    \E\bigg\|\bm{\mu}^\gm\bigg\|^2 
    &\leq \frac{8}{\bigg(|\gD_\gG| - |\gD_\gB|\bigg)^2}\sum_{\vx_i \in \gD_\gG} \E\bigg\|\vx_i\bigg\|^2 + \frac{2\epsilon^2}{\bigg( |\gD_\gG| - |\gD_\gB|\bigg)^2}
\end{flalign}
Since, $\gm$ is {\bf translation equivariant}, we can write:
\begin{flalign}
    \E\bigg[\gm\bigg(\bigg\{\vx_i - \bm{\mu}^\gG | \vx_i \in \gD\bigg\}\bigg)\bigg]
    = \E\bigg[\gm\bigg(\bigg\{\vx_i | \vx_i \in \gD \bigg\}\bigg) - \bm{\mu}^\gG \bigg]
\end{flalign}
Consequently we have that :
\begin{flalign}
    \nonumber
    \E\bigg\|\bm{\mu}^\gm - \bm{\mu}^\gG \bigg\|^2
    &\leq \frac{8}{\bigg(|\gD_\gG| - |\gD_\gB|\bigg)^2}\sum_{\vx_i \in \gD_\gG} \E\bigg\|\vx_i - \bm{\mu}^\gG \bigg\|^2 + \frac{2\epsilon^2}{\bigg( |\gD_\gG| - |\gD_\gB|\bigg)^2}
\end{flalign}
This concludes the proof.
\end{proof}

\subsubsection{Proof of Theorem~\ref{th:gm_match_cvg}}
We restate the theorem for convenience:

{\bf Theorem}~\ref{th:gm_match_cvg} 
Suppose that, we are given, a set of $\alpha$-corrupted samples $\gD$ (\cref{def:corruption_model}), pretrained proxy model $\phi_\mB$, and an $\epsilon$ approx. $\gm(\cdot)$ oracle~\eqref{eq:gm}. Then, $\gmm$ guarantees that the mean of the selected $k$-subset $\bm{\mu}^\gS = \frac{1}{k} \sum_{\vx_i \in \gD_\gS}\vx_i$ converges to the neighborhood of $\bm{\mu}^\gG = \E_{\vx \in \gD_\gG}(\vx)$ at the rate $\gO(\frac{1}{k})$ such that:   
\begin{equation}
    \E\bigg\|\bm{\mu}^\gS - \bm{\mu}^\gG\bigg \|^2\leq 
    \frac{8|\gD_\gG|}{(|\gD_\gG|-|\gD_\gB|)^2}\sum_{\vx \in \gD_\gG}\E\bigg\|\vx -\Mu^\gG \bigg\|^2 +\frac{2\epsilon^2}{(|\gD_\gG|-|\gD_\gB|)^2}
\end{equation}
        
\begin{proof}    

To prove this result, we first show that $\gmm$ converges to $\Mu_\epsilon^\gm$ at $\gO(\frac{1}{k})$. It suffices to show that the error $\delta = \bigg\| \Mu^\gm_\epsilon - \frac{1}{k} \sum_{\vx_i \in \gS}\vx_i\bigg\| \rightarrow 0$ asymptotically. We will follow the proof technique in~\citep{chen2010super} mutatis mutandis to prove this result. 
We also assume that $\gD$ contains the support of the resulting noisy distribution.

We start by defining a $\gm$-centered marginal polytope as the convex hull -- 
\begin{flalign}
    \gM_\epsilon := \text{conv}\bigg\{\vx -\Mu^\gm_\epsilon \;| \vx \in \gD\bigg\}
\end{flalign}

% and thus by invoking the {\bf translation equivariance} property of $\gm$ we get that it also resides in the relative interior of $\gM_\epsilon$.
Then, we can rewrite the update equation~\eqref{eq:theta_update} as: 
\begin{flalign}
    \vtheta_{t+1} 
    &= \vtheta_t + \Mu_\epsilon^\gm - \vx_{t+1}\\
    &= \vtheta_t - (\vx_{t+1} - \Mu_\epsilon^\gm)\\
    &= \vtheta_t - \bigg(\argmax_{\vx \in \gD} \langle \vtheta_t, \vx \rangle - \Mu_\epsilon^\gm\bigg)\\
    &= \vtheta_t - \argmax_{\vm \in \gM_\epsilon} \langle \vtheta_t, \vm \rangle\\
    &= \vtheta_t - \vm_t
\end{flalign}

Now, squaring both sides we get :

\begin{flalign}
    \|\vtheta_{t+1}\|^2 = \|\theta_t\|^2 + \|\vm_t\|^2 - 2 \langle \vtheta_t, \vm_t \rangle
\end{flalign}

rearranging the terms we get:
\begin{flalign}
    \|\vtheta_{t+1}\|^2 - \|\theta_t\|^2 
    &= \|\vm_t\|^2 - 2 \langle \vtheta_t, \vm_t \rangle \\
    &= \|\vm_t\|^2 - 2 \|\vm_t\|\|\vtheta_t\|\langle \frac{\vtheta_t}{\|\vtheta_t\|}, \frac{\vm_t}{\|\vm_t\|} \rangle\\
    &= 2\|\vm_t\|\bigg( \frac{1}{2}\|\vm_t\| - \|\vtheta_t\|\langle \frac{\vtheta_t}{\|\vtheta_t\|}, \frac{\vm_t}{\|\vm_t\|} \rangle \bigg)
\end{flalign}

Assume that $\|\vx_i\| \leq r \;  \forall \vx_i \in \gD$. , Then we note that, 
\begin{flalign*}
    \|\vx_i - \Mu_\epsilon^\gm\| 
    &\leq \|\vx_i\| + \|\Mu_\epsilon^\gm\| \leq 2r
\end{flalign*}
Plugging this in, we get:
\begin{flalign}
    \|\vtheta_{t+1}\|^2 - \|\theta_t\|^2 
    &\leq 2\|\vm_t\| \bigg( r - \|\vtheta_t\|\langle \frac{\vtheta_t}{\|\vtheta_t\|}, \frac{\vm_t}{\|\vm_t\|} \rangle \bigg)
\end{flalign}

Recall that, $\Mu_\epsilon^\gm$ is guaranteed to be in the relative interior of 
$\text{conv}\{\vx \;| \vx \in \gD\}$~\citep{lopuhaa1991breakdown, minsker2015geometric}. Consequently, $\exists \kappa$-ball around $\Mu_\epsilon^\gm$ contained inside $\gM$ and we have $\forall t > 0$
\begin{flalign}
    \langle \frac{\vtheta_t}{\|\vtheta_t\|}, \frac{\vm_t}{\|\vm_t\|} \rangle 
    &\geq \kappa > 0
\end{flalign}
This implies, $\forall t > 0$
\begin{flalign}
    \|\vtheta_t\| \leq \frac{r}{\kappa}
\end{flalign}

Expanding the value of $\vtheta_t$ we have:
\begin{flalign}
    \bigg\|\vtheta_k\bigg\| = \bigg\|\vtheta_0  + k \Mu^\gm_\epsilon - \sum_{i=1}^k \vx_k \bigg\| \leq \frac{r}{\kappa}
\end{flalign}
Apply Cauchy Schwartz inequality:
\begin{flalign}
    \bigg\|k \Mu^\gm_\epsilon - \sum_{i=1}^k \vx_k \bigg\|
    &\leq \bigg\|\vtheta_0\bigg\| + \frac{r}{\kappa} 
\end{flalign}
normalizing both sides by number of iterations $k$
\begin{flalign}
    \bigg\|\Mu^\gm_\epsilon - \frac{1}{k}\sum_{i=1}^k \vx_k \bigg\|
    &\leq \frac{1}{k}\bigg(\bigg\|\vtheta_0\bigg\| + \frac{r}{\kappa}\bigg) 
\end{flalign}
Thus, we have that $\gmm$ converges to $\Mu^\gm_\epsilon$ at the rate $\gO(\frac{1}{k})$. 

Combining this with~\cref{lemma:gm_mean_bound}, completes the proof. 
\end{proof}
\end{document}